\documentclass{article}

    \PassOptionsToPackage{numbers, compress}{natbib}
\usepackage[preprint]{neurips_2026}

\usepackage[utf8]{inputenc} % allow utf-8 input
\usepackage[T1]{fontenc}    % use 8-bit T1 fonts
\usepackage{url}            % simple URL typesetting
\usepackage{booktabs}       % professional-quality tables
\usepackage{amsfonts}       % blackboard math symbols
\usepackage{nicefrac}       % compact symbols for 1/2, etc.
\usepackage{microtype}      % microtypography
\usepackage{xcolor}         % colors

\title{From Local Mismatch to Global Impact: Optimizing Cache Reuse Policy for Efficient Diffusion}

\usepackage[page,toc,titletoc,title]{appendix}

\usepackage{amsmath}
\usepackage{amssymb}
\usepackage{amsfonts}
\usepackage{amsthm}
\usepackage{mathtools}
\usepackage{bm}

\usepackage{graphicx}
\usepackage{wrapfig}
\usepackage{subcaption}
\usepackage{booktabs}
\usepackage{multirow}
\usepackage{makecell}
\usepackage{array}

\usepackage{algorithm}
\usepackage{algorithmic}

\usepackage{nicefrac}
\usepackage{microtype}
\usepackage{enumitem}
\usepackage{pifont}
\usepackage{tikz}
\usepackage{xspace}
\usepackage{makecell}

\usepackage[dvipsnames,svgnames,table]{xcolor}

\definecolor{citecolor}{HTML}{2980b9}
\definecolor{linkcolor}{HTML}{c0392b}
\definecolor{mygreen}{RGB}{46,125,50}

\usepackage[
    colorlinks,
    citecolor=citecolor,
    linkcolor=linkcolor
]{hyperref}

\usepackage[capitalize,noabbrev]{cleveref}

\setitemize{
    leftmargin=1.2em, topsep=-0.5pt, itemsep=3pt, parsep=3pt
}

\theoremstyle{plain}
\newtheorem{theorem}{Theorem}[section]

\newtheorem{lemma}[theorem]{Lemma}

\theoremstyle{definition}

\newtheorem{assumption}[theorem]{Assumption}

\theoremstyle{remark}

\theoremstyle{plain}
\newtheorem{manualtheoreminner}{Theorem}
\newenvironment{manualtheorem}[1]{%
  \renewcommand\themanualtheoreminner{#1}%
  \manualtheoreminner
}{%
  \endmanualtheoreminner
}

\makeatletter
\DeclareRobustCommand\onedot{%
    \futurelet\@let@token\@onedot
}

\def\@onedot{%
    \ifx\@let@token.\else.\null\fi\xspace
}

\makeatother

\author{%
  \textbf{Xichen Ye}$^{1}$, \textbf{Yifan Wu}$^{2}$, \textbf{Zhikang Xie}$^{1}$ \\ \textbf{Xiangyu Yue}$^{2}$, \textbf{Cheng Jin}$^{1}$, \textbf{Weizhong Zhang}$^{3,}$\thanks{Corresponding author} \\[3pt]
  $^{1}$College of Computer Science and Artificial Intelligence, Fudan University\\
  $^{2}$Faculty of Engineering, The Chinese University of Hong Kong \\
  $^{3}$School of Data Science, Fudan University \\
}

\begin{document}

\maketitle

\begin{abstract}
    Diffusion models have achieved dominant performance in visual generation but suffer from substantial inference overhead.
    While cache-based acceleration has emerged as a promising solution, existing policies rely on local similarity heuristics, which we identify as being significantly misaligned with final generation quality.
    This discrepancy stems from the non-uniform propagation and accumulation of errors along the denoising trajectory.
    To address this, we propose \textbf{G}lobal-Impact \textbf{Cache} (\textbf{GCache}).
    We first establish a rigorous theoretical characterization of the error propagation upper bound.
    Recognizing that this bound can be overly conservative for complex, highly non-convex diffusion models, we further reparameterize the propagation exponent with a Bernstein form and reformulate cache policy search as a bilevel optimization problem.
    In detail, GCache identifies an optimal reuse policy in the inner objective while aligning the error-weighting function with generation quality loss in the outer objective.
    This framework effectively reconciles theoretical rigor with empirical performance, learning to prioritize computation where it most impacts visual fidelity.
    Extensive experiments demonstrate that GCache consistently outperforms prior caching strategies on both video and image generation.
    Notably, on the state-of-the-art Wan2.1 video diffusion model, GCache maintains a $2.17\times$ speedup while significantly enhancing generation quality, reducing LPIPS from $0.1095$ to $0.0316$.
\end{abstract}

\addtocontents{toc}{\protect\setcounter{tocdepth}{-1}}

\section{Introduction}

% 1. 【diffusion介绍】diffusion 好 // 累
In recent years, diffusion models~\cite{DBLP:conf/nips20:DDPM, DBLP:conf/iclr21:ScoreSDE, DBLP:conf/nips21:DiffusionBeatGAN, DBLP:conf/iclr23:FlowMatching} have emerged as the dominant paradigm in visual generation, delivering high-fidelity and diverse outputs across various modalities, including images~\cite{DBLP:conf/cvpr22:StableDiffussion, DBLP:conf/nips22:imagen} and videos~\cite{DBLP:journals/corr23/StableVideoDiffusion, DBLP:conf/cvpr24:VideoCrafter2}.
% In recent years, diffusion models~\cite{DBLP:conf/nips20:DDPM, DBLP:conf/iclr21:ScoreSDE, DBLP:conf/nips21:DiffusionBeatGAN, DBLP:conf/iclr23:FlowMatching} have become a dominant paradigm in visual generation, delivering high-fidelity and diverse outputs across images~\cite{DBLP:conf/cvpr22:StableDiffussion, DBLP:conf/nips22:imagen}, videos~\cite{DBLP:journals/corr23/StableVideoDiffusion, DBLP:conf/cvpr24:VideoCrafter2}, etc.
% \YXC{
% Driven by the demand for superior quality and enhanced scalability, the architectural backbone has evolved from the lightweight U-Net~\cite{DBLP:conf/miccai15:Unet} to the Diffusion Transformer (DiT)~\cite{DBLP:conf/iccv23:DiT}.
% This transition has enabled significant capacity scaling and catalyzed rapid progress in large-scale diffusion generation, with particularly notable breakthroughs in video synthesis~\cite{DBLP:journals/corr24:opensora, DBLP:conf/iclr25:CogVideoX, DBLP:journals/corr25:wan2.1}.
% }
% Motivated by the pursuit of higher generation quality and stronger scalability, diffusion models have evolved from lightweight U-Net~\cite{DBLP:conf/miccai15:Unet} to large-scale DiT~\cite{DBLP:conf/iccv23:DiT}, enabling improved capacity scaling and driving rapid progress in large-scale diffusion generation, with particularly notable gains in video synthesis~\cite{DBLP:journals/corr24:opensora, DBLP:conf/iclr25:CogVideoX, DBLP:journals/corr25:wan2.1}.
Despite these successes, diffusion models remain burdened by substantial inference overhead, stemming from the iterative nature of solving the underlying Ordinary Differential Equations (ODEs), which necessitates numerous model evaluations for a single sample.
% Despite their successes, diffusion models still suffer from substantial inference overhead, which requires solving a corresponding ODE and requires multiple evaluations for generating a single sample.
% Despite their effectiveness, diffusion models still suffer from substantial inference overhead, which severely hinders practical deployment, especially for high-resolution and long-duration video generation.
To mitigate this limitation, various acceleration mechanisms have been explored from multiple directions, including model-centric compression
% model-based
~\cite{DBLP:conf/nips23:PTQD, DBLP:conf/cvpr24:LD-Pruner},
advanced sampling solvers
% solver-based
~\cite{DBLP:conf/nips22:DPM-Solver, DBLP:conf/iclr23:FastSampling, DBLP:conf/nips23:UniPC},
and, more recently, cache-based mechanisms
% and cache-based acceleration
~\cite{DBLP:journals/corr24:Delta-DiT, DBLP:conf/cvpr25:TeaCache, DBLP:journals/corr25:ERTACache}.

% \begin{figure}[t]
%     \centering
%     \vskip 0.1in
%     \includegraphics[width=0.5\columnwidth]{figures/intro_2.pdf}
%     \vskip -0.05in
%     \caption{
%         Comparison between \textcolor{blue}{local mismatch} and \textcolor{red}{global impact} (lower is better).
%         We conduct a series of independent experiments in which a cached residual is reused at exactly one specific timestep.
%         For each timestep, the blue marker (\textcolor{blue}{Rel-L1}) measures the local discrepancy between the ground-truth residual and the cached residual from the preceding step, while the red marker (\textcolor{red}{LPIPS}) reflects the resulting impact on final generation quality.
%         As shown, local mismatch is not a reliable proxy for global impact, as large Rel-L1 values around step~19 and in the final denoising stages correspond to only minor increases in LPIPS, i.e., a relatively small degradation in final output quality.
%         Note that the prominent Rel-L1 spike around step~19 is a known characteristic of Flux-dev~1.0 when using the Euler ODE solver.
%         Similar spike behaviors in different diffusion models have also been reported in prior work~\cite{DBLP:conf/cvpr25:TeaCache}.
%     }
%     \label{fig:intro}
%     \vskip -0.1in
% \end{figure}

% 2. 【cache介绍】cache 简易的方法
Among current acceleration strategies, cache-based methods provide a practical approach to speeding up diffusion inference.
Unlike model-centric techniques, caching avoids intensive retraining or distillation of model parameters and remains orthogonal to advanced sampling solvers, making it a highly lightweight solution for denoising acceleration.
The primary objective is to establish an inference-time policy that identifies redundant intermediate residuals across adjacent timesteps to avoid repeated computations.
While initial approaches~\cite{DBLP:journals/corr24:Delta-DiT, DBLP:journals/corr24/Fora, DBLP:conf/iclr25:PAB} relied on uniform reuse schedules, they lacked the flexibility to adapt to varying residual dynamics across different timesteps.
Consequently, recent works~\cite{DBLP:conf/cvpr25:TeaCache, DBLP:journals/corr25:ERTACache} have shifted toward non-uniform strategies.
These methods employ local similarity metrics to quantify the mismatch between target and cached residuals, triggering reuse only when the discrepancy is sufficiently small.

\begin{wrapfigure}{r}{0.5\linewidth}
    \centering
    \vskip -0.2in
    \includegraphics[width=0.95\linewidth]{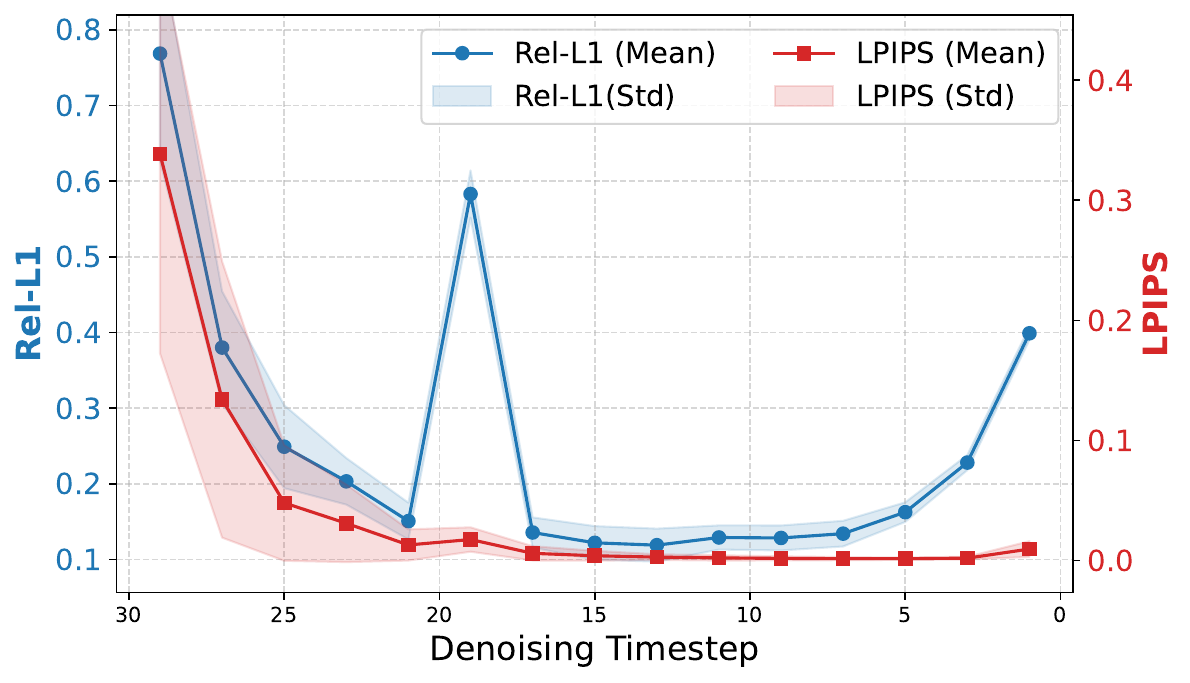}
    \vskip -0.075in
    \caption{\footnotesize
        Comparison between \textcolor{blue}{local mismatch} and \textcolor{red}{global impact} (lower is better).
        We conduct a series of independent experiments in which a cached residual is reused at exactly one specific timestep.
        For each timestep, the blue marker (\textcolor{blue}{Rel-L1}) measures the local discrepancy between the ground-truth residual and the cached residual from the preceding step, while the red marker (\textcolor{red}{LPIPS}) reflects the resulting impact on final generation quality.
        % As shown, local mismatch is not a reliable proxy for global impact, as large Rel-L1 values around step~19 and in the final denoising stages correspond to only minor increases in LPIPS, i.e., a relatively small degradation in final output quality.
        % Note that the prominent Rel-L1 spike around step~19 is a known characteristic of Flux-dev~1.0 when using the Euler ODE solver.
        % Similar spike behaviors in different diffusion models have also been reported in prior work~\cite{DBLP:conf/cvpr25:TeaCache}.
    }
    \label{fig:intro}
    \vskip -0.2in
\end{wrapfigure}

Despite these advancements, existing approaches remain limited by their reliance on local similarity metrics, such as the relative $\ell_1$ distance.
However, empirical evidence suggests that such locally-based strategies fail to accurately capture how individual cache reuse decisions affect final generation quality.
As illustrated in \cref{fig:intro}, large local discrepancies may correspond to only minor perceptual degradation: pronounced relative $\ell_1$ spikes occur around step~19\footnote{The prominent Rel-L1 spike around step~19 is a known characteristic of Flux-dev~1.0 under the Euler ODE solver. Similar spike behaviors across diffusion models have also been reported in prior work~\cite{DBLP:conf/cvpr25:TeaCache}.} and increase sharply in the final denoising stages, yet result in only marginal increases in LPIPS.
This mismatch reveals that local discrepancy alone is an unreliable proxy for global generation impact, indicating that effective cache reuse policies must account for both local reuse errors and their cumulative impact on final generation quality.
To this end, we first establish a theoretical characterization of the error propagation dynamics in cache-reused diffusion trajectories.
This analysis provides a principled foundation by relating local discrepancies to a global cumulative error bound, offering a rigorous heuristic for policy optimization.
While this analytical bound serves as a robust guide, its worst-case nature inherently introduces a pessimistic bias, as it does not fully account for the high non-convexity and intrinsic error-resilience of diffusion models.
To bridge this gap and unlock the full potential of our theoretical framework, we propose \textbf{Global-Impact Cache} (\textbf{GCache}).
In detail, we parameterize the propagation exponent in the theoretical bound as a Bernstein form and formulate the policy search as a bilevel optimization problem, where the inner objective identifies the optimal reuse policy that minimizes the current error bound for a given set of parameters, while the outer objective optimizes the propagation parameters by minimizing empirical generation loss.
In this way, GCache learns an error-weighting function that better reflects empirical error propagation, leading to more informed cache reuse decisions and improved generation quality.
% To this end, we first provide a theoretical characterization that upper-bounds the final generation error induced by cache reuse under the probability flow ODE formulation. 
% While this bound offers a principled way to relate step-wise reuse errors to the final outcome, it can be conservative in practice.
% We therefore formulate cache selection as a bilevel optimization problem that tightens this error control by directly optimizing the reuse policy under a compute budget, leading to a more effective cache schedule without introducing additional rectification modules or extra inference steps.
% Extensive experiments on both image and video generation across multiple backbones demonstrate consistent improvements over prior caching strategies.
% Notably, on Wan2.1, our method achieves the same $2.17\times$ speedup while substantially improving generation quality (LPIPS $0.1095 \rightarrow 0.0316$).
Extensive experiments across multiple image and video generation models demonstrate that our method, GCache, consistently outperforms prior caching strategies.
% achieving superior speed--quality trade-offs without introducing any additional inference-time computation.
% Notably, on Wan2.1, our approach maintains a $2.17\times$ speedup while significantly enhancing generation quality, reducing the LPIPS score from $0.1095$ to $0.0316$.
Overall, we summarize our contributions as follows:
\begin{itemize}
    [leftmargin=1.2em, topsep=-0.5pt, itemsep=3pt, parsep=3pt]
    \item We identify a fundamental misalignment between local cache reuse errors and global generation quality, and provide a theoretical analysis that characterizes how cache reuse errors propagate during diffusion sampling, establishing an analytical upper bound on their global impact.
    % \item We demonstrate that, while theoretically sound, this analytical upper bound does not act as a tight proxy for empirical error propagation, due to the highly non-convex nature of the diffusion models.
    \item We demonstrate that there is an estimative gap between this theoretical bound and empirical error behavior, observing that the bound's worst-case assumptions can be conservatively biased in the context of highly non-convex diffusion dynamics.
    \item Motivated by this, we propose GCache, which reparameterizes the error propagation exponent with a Bernstein form and formulates cache reuse policy search as a bilevel optimization problem to better align theoretical estimates with empirical error behavior.
    \item We conduct extensive experiments on both image and video generation tasks, demonstrating that GCache consistently achieves superior speed--quality trade-offs compared to existing cache-based acceleration methods.
\end{itemize}

\section{Preliminaries}

In this section, we briefly review the background on diffusion models and cache reuse to establish notation and context.
Further discussions of related work are provided in Appendix~\ref{app:rel_work}.

\subsection{Diffusion Models}

Diffusion models~\cite{DBLP:conf/nips20:DDPM, DBLP:conf/iclr21:ScoreSDE, DBLP:conf/iclr23:FlowMatching} synthesize samples by gradually perturbing real data $\bm x \sim p_{\text{data}}(\bm x)$ into prior noise $\bm n \sim p_{\text{prior}}(\bm n)$ (e.g., a standard Gaussian distribution $\mathcal{N}(\bm 0, \mathbf{I})$), subsequently learning to invert this process for sample generation.

Recent works~\cite{DBLP:conf/icml24:SD3} frequently adopt the Flow Matching paradigm~\cite{DBLP:conf/iclr23:FlowMatching}.
Specifically, given a data-noise pair $(\bm x, \bm n)$, a flow path is constructed as $\bm x_t = (1 - t) \bm x + t \bm n$ for $t \in [0, 1]$, which induces the conditional velocity field:
\begin{equation}
    \bm v_t(\bm x_t \mid \bm x) = \frac{\mathrm{d}}{\mathrm{d}t} \left[ (1 - t) \bm x + t \bm n \right] = \bm n - \bm x.
\end{equation}
Since a specific $\bm x_t$ can result from various $(\bm x, \bm n)$ pairs, Flow Matching targets the marginal velocity field:
\begin{equation}
    \bm v(\bm x_t, t) := \mathbb{E}_{p_t(\bm x \mid \bm x_t)} \left[ \bm v_t(\bm x_t \mid \bm x) \right].
\end{equation}
A neural network $\bm v_\theta$ is trained to approximate this marginal field by minimizing the conditional Flow Matching loss:
\begin{equation}
    \mathcal{L}_\text{CFM}(\theta) = \mathbb{E}_{t, \bm x, \bm n} \left\| \bm v_\theta(\bm x_t, t) - \bm v_t(\bm x_t \mid \bm x) \right\|_2^2.
\end{equation}
This objective is equivalent to minimizing the Flow Matching loss $\mathcal{L}_\text{FM}(\theta) = \mathbb{E}_{t, p_t(\bm x_t)} \| \bm v_\theta (\bm x_t, t) - \bm v(\bm x_t, t) \|_2^2$, effectively fitting the model to the true marginal velocity.

During inference, samples are generated by solving the corresponding ordinary differential equation (ODE):
\begin{equation}
    \frac{\mathrm{d}}{\mathrm{d}t} \bm x_t = \bm v(\bm x_t, t),
\end{equation}
initialized at $\bm x_1 \sim p_\text{prior}$ and integrated backward to $t = 0$.
The final sample is given by $\bm x_0 = \bm x_1 - \int_0^1 \bm v(\bm x_\tau, \tau) \mathrm{d}\tau$.
In practice, this integral is approximated using numerical ODE solvers.
For instance, the Euler method discretizes the trajectory over $N$ timesteps $1 = t_N > t_{N-1} > \cdots > t_2 > t_1 = 0$, computing each step as:
\begin{equation}
    \bm x_{t_{i}} = \bm x_{t_{i+1}} - (t_i - t_{i+1}) \bm v(\bm x_{t_{i+1}}, t_{i+1}).
\end{equation}

\subsection{Cache Reuse}

The primary inference bottleneck in diffusion models arises from the repetitive model evaluations necessitated by the ODE solver.
Cache reuse has emerged as a promising strategy to mitigate this overhead by leveraging the temporal redundancy of intermediate representations across sampling steps.
A prevalent approach involves reusing residual mappings across specific model layers~\cite{DBLP:journals/corr24:Delta-DiT}.

Formally, consider a feed-forward neural network (e.g., a Transformer~\cite{DBLP:conf/nips17:Attention}) decomposed as:
\begin{equation}
    \bm v(\bm x_t, t) = (f_{\text{out}} \circ f_{\text{mid}} \circ f_{\text{in}}) (\bm x_t, t),
\end{equation}
where $f_{\text{mid}}$ represents the intermediate blocks (e.g., core transformer layers), while $f_{\text{in}}$ and $f_{\text{out}}$ denote the layers responsible for input embedding and output projection (e.g., patchify and unpatchify), respectively.
Let $\bm h_t := f_{\text{in}}(\bm x_t)$ and $\bm z_t := f_{\text{mid}}(\bm h_t)$, the residual is then defined as
\begin{equation}
    \bm \delta_t := \bm z_t - \bm h_t.
\end{equation}
During the sampling process, a cache reuse policy $\bm m \in \{0, 1\}^{N}$ governs the use of a stored residual $\bm \delta^c$.
At each timestep $t_n$, the policy determines whether to approximate the state by reusing the cached residual ($\bm m_{t_i} = 0$):
\begin{equation}
\bm z^c_{t_i} = \bm \delta^c + \bm h_{t_i},
\end{equation}
or to perform a full computation and update the cache ($\bm m_{t_i} = 1$):
\begin{equation}
\bm \delta^c \leftarrow \bm \delta_{t_i}.
\end{equation}
The objective of a cache reuse policy is to minimize the degradation of generation quality subject to a specific computational budget, which is typically defined by the total number of cache reuses, $\| \bm m \|_0$.

To develop effective caching policies, recent studies~\cite{DBLP:conf/cvpr25:TeaCache, DBLP:journals/corr25:ERTACache} typically employ a local error metric, specifically the relative $\ell_1$ distance, to determine whether to reuse the previously cached residual $\bm \delta^c$ at timestep $t_i$:
\begin{equation}
    d(\bm \delta^c, \bm \delta_{t_i}) = \frac{\| \bm \delta^c - \bm \delta_{t_i} \|_1}{\| \bm \delta^c \|_1}.
\end{equation}
However, as we demonstrate below, this local-only perspective fails to account for the error propagation inherent in the denoising process, ultimately resulting in sub-optimal caching decisions.

\section{Method}

In this section, we begin with preliminary experiments that reveal the error propagation behavior in the denoising process (\cref{sec:method:motivation}).
We then develop a theoretical upper bound that characterizes how local cache reuse errors propagate across timesteps (\cref{sec:method:theory}).
Since directly optimizing cache reuse policies with this bound proves insufficient in practice, we further refine it via a bilevel optimization framework and propose \textbf{G}lobal-Impact \textbf{Cache} (\textbf{GCache}) with an efficient solution strategy (\cref{sec:method:method}).

% In this section, we first conduct a preliminary experiment to reveal the error propagation phenomenon inherent in the denoising process (\cref{sec:method:motivation}).
% Building on these empirical insights, we then provide a rigorous theoretical analysis of how local errors induced by cache reuse propagate throughout the generation trajectory (\cref{sec:method:theory}).
% Finally, we propose Global-Impact Cache (GCache), which formulates the search for an optimal caching policy as a bilevel optimization problem and discusses how to effectively solve it (\cref{sec:method:method}).
% \YXC{TODO}

\subsection{Motivation}
\label{sec:method:motivation}

\begin{wrapfigure}{r}{0.46\linewidth}
    \centering
    \vskip -0.4in
    \includegraphics[width=0.95\linewidth]{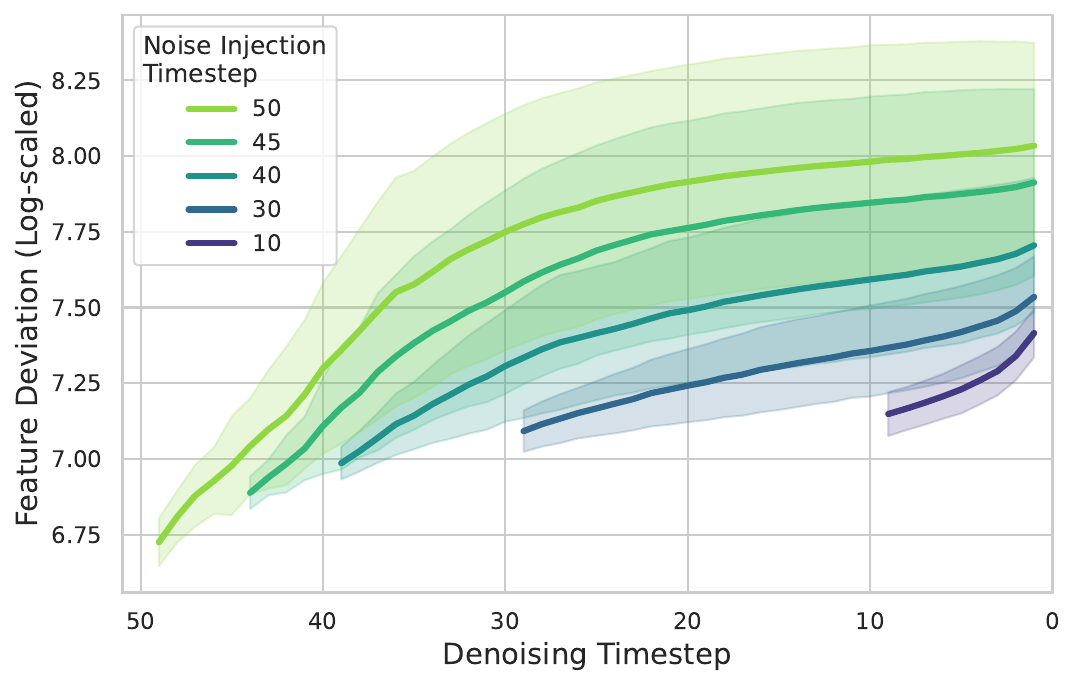}
    \vskip -0.07in
    \caption{\footnotesize
        Characterization of error propagation across the denoising trajectory.
        % We inject fixed-magnitude perturbations into intermediate features at different timesteps (shown in different colors) and track the log-scaled deviation $\|\bm x_t - \hat{\bm x}_t\|_1$ over subsequent denoising steps, where $\bm x_t$ and $\hat{\bm x}_t$ denote the unperturbed and perturbed states.
        % Perturbations introduced earlier are progressively amplified and lead to substantially larger final deviations than those injected later.
        % , revealing a cumulative and time-dependent error propagation effect.
    }
    \label{fig:motivation}
    \vskip -0.25in
\end{wrapfigure}
% To motivate our study, we first conduct experiments on video generation.
% Specifically, we inject random noise of a fixed magnitude into the intermediate feature maps at various timesteps and measure the resulting deviation between the perturbed and unperturbed trajectories throughout the subsequent denoising process.
% As illustrated in \cref{fig:motivation}, perturbations introduced at earlier stages undergo progressive amplification, ultimately exerting a disproportionate impact on the final output.
To motivate our study, we first conduct controlled perturbation experiments on video generation.
Specifically, we inject random noise of a fixed magnitude into intermediate feature maps at different denoising timesteps and track the log-scaled deviation $\|\bm x_t - \hat{\bm x}_t\|_1$ over subsequent denoising steps, where $\bm x_t$ and $\hat{\bm x}_t$ denote the unperturbed and perturbed states.
As illustrated in \cref{fig:motivation}, Perturbations introduced earlier are progressively amplified and lead to substantially larger final deviations than those injected later.

These findings reveal a cumulative and time-dependent error propagation effect, suggesting that an effective cache reuse policy should account for not only the local cache reuse error but also the error propagation inherent in the denoising process.
Building on this observation, the following subsection establishes a formal relationship between local reuse discrepancies and their cumulative propagation.

% \cref{fig:motivation noise_timestep_vs_quality} illustrates the relationship between generation quality and the timing of noise injection, revealing that earlier perturbations lead to more significant quality degradation.
% Furthermore, \cref{fig:motivation error propagation}

% demonstrates that noise introduced at early stages undergoes progressive amplification during the denoising process, ultimately culminating in a more pronounced impact on the final output.

\subsection{Error Propagation of Cache Reuse}
\label{sec:method:theory}

To rigorously quantify how local cache reuse errors accumulate across the denoising trajectory, we first establish a formal theoretical framework.
Building on the empirical insights from \cref{sec:method:motivation}, we aim to derive an upper bound for the final generation error that accounts for the time-dependent nature of error propagation.
% To ensure the mathematical tractability of the probability flow and the stability of the numerical solver,
We first introduce the following regularity conditions, which are standard in the analysis of diffusion-based generative models~\cite{DBLP:conf/icml23:Assumption}.

\begin{assumption}[Regularity and Approximation Conditions]
    \label{assump:regularity}
    For all denoising timesteps $t \in [t_1, t_N]$, we assume the following conditions hold:
    % \begin{enumerate}[label=(\roman*), nosep]
    \begin{enumerate}[
        label=(\roman*),
        topsep=0pt,
        itemsep=3pt,
        parsep=3pt,
        partopsep=3pt,
        leftmargin=2em
    ]
    % \begin{enumerate}
    % \item \textbf{Bounded Velocity Approximation Error:}
    % The learned velocity $\bm v_\theta$ approximates the ground-truth marginal velocity with a bounded local error $\eta \ge 0$ such that $\left\| \bm v_\theta(\bm x, t) - \bm v(\bm x, t) \right\|_1 \le \eta$ for all $\bm x$ and $t$.
    \item \textbf{Bounded Approximation Error:}
    The empirical velocity field $\bm v_\theta$ approximates the ground-truth marginal velocity $\bm v$ with a uniform error bound $\eta \ge 0$, i.e., $\left\| \bm v(\bm x, t) - \bm v_\theta(\bm x, t) \right\|_1 \le \eta$ for all $\bm x$ and $t$.
    % \item \textbf{Lipschitz Continuity:} The output projection $f_\text{out}(\cdot)$ is $L_\text{out}$-Lipschitz continuous, and the ground-truth marginal velocity is uniformly $L$-Lipschitz continuous, i.e., $\| \bm v(\bm x, t) - \bm v(\bm y, t) \|_1 \le L \|\bm x - \bm y\|_1$ for all $t$.
    \item \textbf{Lipschitz Continuity:}
    The output projection $f_{\text{out}}$ is $L_{\text{out}}$-Lipschitz continuous, and the learned velocity field $\bm v_\theta$ is $L_t$-Lipschitz continuous with respect to its first argument, i.e., $\| \bm v_\theta(\bm x, t) - \bm v_\theta(\bm y, t) \|_1 \le L_t \|\bm x - \bm y\|_1$ for all $\bm x, \bm y, t$, where $L_t \le L$ for all $t$.
    % \item \textbf{Bounded Velocity Derivative:} The total derivative of the ground-truth marginal velocity with respect to time is bounded by $M$ such that $\left\| \frac{\mathrm{d}}{\mathrm{d} t} \left( \bm v (\bm x_t, t) \right) \right\|_1 \le M$.
    \item \textbf{Bounded Velocity Dynamics:}
    The total time derivative of the ground-truth velocity field along any trajectory $\bm x_t$ is bounded by $M$, i.e., $\left\| \frac{\mathrm{d}}{\mathrm{d} t} \bm v (\bm x_t, t) \right\|_1 \le M$.
    \end{enumerate}
\end{assumption}

% \YXC{TODO}

% To investigate how reusing cached residuals impacts final generation fidelity,
We begin by analyzing a scenario involving a single reuse event at timestep $t_{i+1}$.
Specifically, we evaluate the final error relative to the ground-truth ODE trajectory under Euler discretization.

\begin{theorem}[Global Error Bound under Cache Reuse]
    \label{theorem:error-bound_single-step_PF-ODE}
    Under Assumption~\ref{assump:regularity}, let $\bm{x}_{t}$ be the ground-truth ODE state at time $t$, and let $\hat{\bm{x}}^c_t$ be the state generated by the Euler solver using the learned velocity $\bm{v}_\theta$, incorporating a single cache reuse event.
    Specifically, suppose that at step $t_{i+1} \to t_i$, we reuse a cached residual $\bm{\delta}_{t_{i+1}}^c$ with error $\bm{\epsilon}_{t_{i+1}}^c = \bm{\delta}_{t_{i+1}} - \bm{\delta}^c_{t_{i+1}}$.
    Given a sequence of timesteps $t_{N}, t_{N-1}, \dots, t_1$ with uniform step size $h$, the global error at the final timestep $t_1$ is bounded by:
    \begin{equation}
        \left\| \bm{x}_{t_1} - \hat{\bm x}^c_{t_1} \right\|_1
        \le \underbrace{ h L_\text{out} \left\| \bm \epsilon^c_{t_{i+1}} \right\|_1 e^{(i-1)hL} }_{\text{Cache Reuse Error}}
        + \underbrace{ \left( \frac{\eta}{L} + \frac{h M}{2 L} \right) \left( e^{(N-1) h L} - 1\right) }_{ \text{Approximation \& Discretization Error} }.
    \end{equation}
\end{theorem}

\cref{theorem:error-bound_single-step_PF-ODE} demonstrates that the final error can be decomposed into three distinct components:
(i) the propagated error originating from cache reuse,
(ii) the network approximation error stemming from the learned velocity,
and (iii) the discretization error inherent to the ODE solver.

Notably, the cache reuse term is scaled by an exponential factor, indicating that local errors introduced by cache reuse are amplified by the system dynamics as the trajectory evolves.
This suggests that the impact of cache reuse should not be judged solely by the magnitude of the local error $\bm \epsilon_{t_{i+1}^c}$, but rather by its ``positional'' impact:
errors introduced earlier in the denoising process (larger $i$) may lead to significantly larger deviations in the final generated sample.
This theoretical insight corroborates the empirical observations presented in \cref{fig:motivation}.

In practice, we are primarily concerned with the deviation of the cached trajectory from the baseline discretization rather than from the ideal ODE solution.
Our objective is to enhance the efficiency of diffusion sampling while preserving its generative fidelity as faithfully as possible.
To this end, we establish error propagation bounds for single- and multi-step cache reuse in the following theorems.

\begin{theorem}[Error Propagation Bound under Single-Step Cache Reuse]
    \label{theorem:error-bound_single-step_Euler}
    Under Assumption~\ref{assump:regularity}, let $\hat{\bm{x}}_t$ be the state generated by the Euler solver using the learned velocity $\bm{v}_\theta$, and let $\hat{\bm{x}}_t^c$ be the state incorporating a single cache reuse event at step $t_{i+1} \to t_i$. 
    Suppose the reuse of a cached residual $\bm{\delta}_{t_{i+1}}^c$ introduces an error $\bm{\epsilon}_{t_{i+1}}^c = \bm{\delta}_{t_{i+1}} - \bm{\delta}^c_{t_{i+1}}$.
    Given a sequence of timesteps $t_{N}, t_{N-1}, \dots, t_1$ with step sizes $h_{t_{n+1}} = t_{n+1} - t_n$, the cumulative error at the final timestep $t_1$ is bounded by:
    \begin{equation}
        \left\| \hat{\bm x}_{t_1} - \hat{\bm x}^c_{t_1} \right\|_1
        \le \left\| \bm \epsilon^c_{t_{i+1}} \right\|_1 e^{w_{t_{i+1}}},
    \end{equation}
    where the propagation exponent $w_{t_{i+1}}$ is defined as:
    \begin{equation}
        w_{t_{i+1}}
        = \ln \left( h_{t_{i+1}} L_{\text{out}} \right)
        + \sum_{n=1}^{i-1} h_{t_{n+1}} L_{t_{n+1}}.
    \end{equation}
\end{theorem}

\begin{theorem}[Error Propagation Bound under Multi-Step Cache Reuse]
    \label{theorem:error-bound_multi-step_Euler}
    Under Assumption~\ref{assump:regularity}, let $\hat{\bm{x}}_t$ be the state generated by the Euler solver using the learned velocity $\bm{v}_\theta$, and let $\hat{\bm{x}}_t^c$ be the state incorporating multiple cache reuse events at every step $t_{i+1} \to t_i$ for $i \in \{N-1, \dots, 1\}$. 
    Suppose each reuse of a cached residual introduces a local error $\bm{\epsilon}_{t_{i+1}}^c = \bm{\delta}_{t_{i+1}} - \bm{\delta}^c_{t_{i+1}}$.
    Given a sequence of timesteps $t_{N}, t_{N-1}, \dots, t_1$ with step sizes $h_{t_{i+1}} = t_{i+1} - t_i$, the cumulative error at the final timestep $t_1$ is bounded by:
    \begin{equation}
        \label{eq:theoretical_prop_error_ub}
        \left\| \hat{\bm{x}}_{t_1} - \hat{\bm{x}}^c_{t_1} \right\|_1
        \le \sum_{i=1}^{N-1} \left\| \bm \epsilon^c_{t_{i+1}} \right\|_1 e^{w_{t_{i+1}}}.
    \end{equation}
    where the propagation exponent $w_{t_{n+1}}$ is defined as:
    \begin{equation}
        w_{t_{i+1}}
        = \ln \left( h_{t_{i+1}} L_\text{out} \right)
        + \sum_{j=1}^{i-1} h_{t_{j+1}} L_{t_{j+1}}.
    \end{equation}
\end{theorem}

\cref{theorem:error-bound_single-step_Euler} and \cref{theorem:error-bound_multi-step_Euler} provide an analytical characterization of how a cached trajectory deviates from the baseline discretization.
Crucially, these bounds suggest that an optimal cache reuse policy $\bm m$ can be identified by minimizing the upper bound of the total propagation error:
\begin{equation}
    \label{eq:optim_theoretical_ub}
    \bm m^\star(\bm s)
    \in \underset{\bm m \in \mathcal{C}}{\arg \min} \sum_{i=1}^{N-1} \left\| \bm \epsilon^c_{t_{i+1}} \right\|_1 e^{w_{t_{i+1}}},
\end{equation}
where $\mathcal{C} = \{ \bm m: \| \bm m \|_0 =K \}$ denotes the feasible set of policies under a predefined budget $K$, representing the number of full-computation steps allowed throughout the denoising process.

\begin{figure}[tbp]
    \centering
    % 第一张子图
    \begin{subfigure}[b]{0.49\textwidth}
        \centering
        \includegraphics[width=0.9\textwidth]{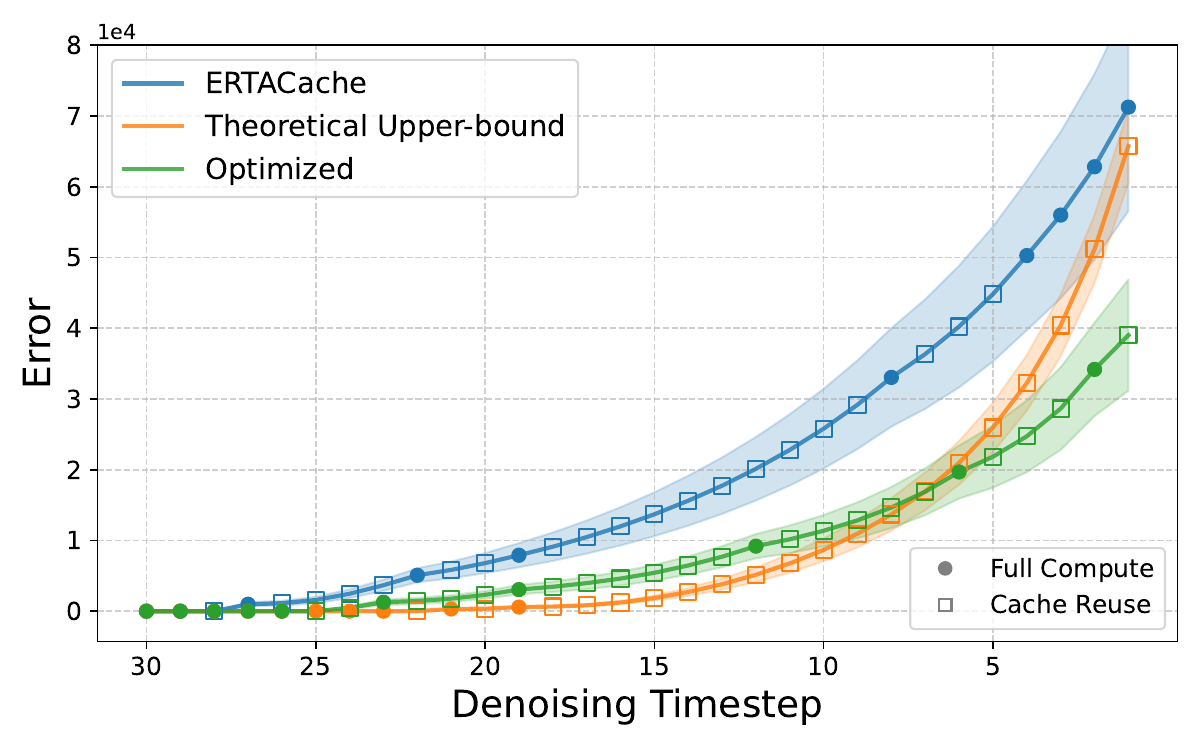}
        \caption{Comparison of different cache policies.}
        \label{fig:comparison_cache_policies}
    \end{subfigure}
    \hfill % 在两个子图之间添加弹性间距，使它们分别向左右对齐
    % 第二张子图
    \begin{subfigure}[b]{0.49\textwidth}
        \centering
        \includegraphics[width=0.9\textwidth]{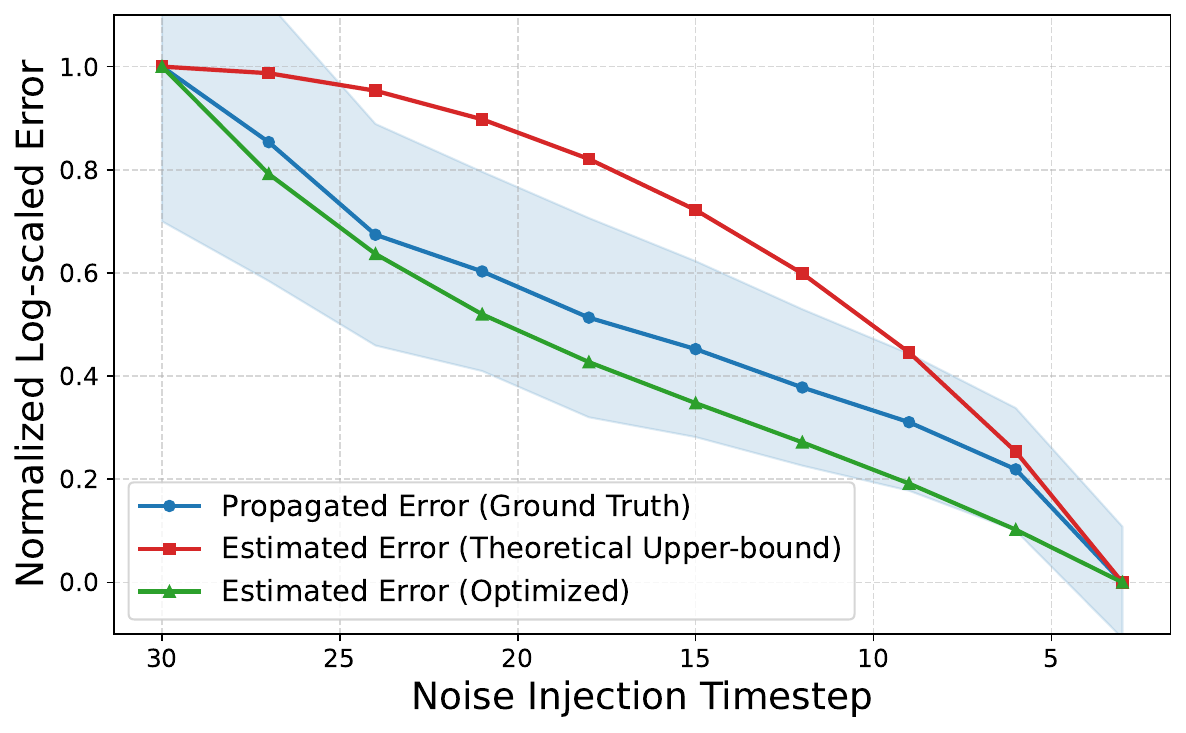}
        \caption{Comparison of error propagation estimations}
        \label{fig:estimated_error}
    \end{subfigure}
    
    \caption{
         Analysis of cache policies and error estimation.
        (a) Policy comparison.
        Reconstruction error across the denoising process.
        The theoretical policy (\textcolor{orange}{orange}) is overly pessimistic in early stages, leading to late-stage error spikes; GCache (\textcolor{mygreen}{green}) balances computation to achieve the lowest final error.
        (b) Estimation fidelity.
        Comparison of error propagation profiles.
        The analytical bound (\textcolor{red}{red}) provides a conservative upper limit, while GCache's optimized weighting (\textcolor{mygreen}{green}) aligns tightly with the empirical ground truth (\textcolor{blue}{blue}).
        % Comparison of caching policies and error propagation estimation.
        % (a) Reconstruction error during the denosing process,
        % where circles denote full computation and squares denote cache reuse.
        % The policy derived from the theoretical upper bound (orange) allocates most computation to early steps, leading to excessive reuse and a sharp error spike near the end, while our optimized policy (green) balances computation across timesteps and achieves the lowest final error.
        % (b) Comparison between empirical error propagation (blue) and theoretical estimation (red) when injecting a perturbation at a single timestep.
        % The analytical bound consistently overestimates the true impact, especially at early stages, whereas our optimized weighting (green) closely matches the empirical behavior.
    }
    \label{fig:cache_policy_error_propagation}
\end{figure}

\textbf{Effectiveness of the Analytical Upper Bound.}
% \cref{theorem:error-bound_multi-step_Euler} provides a principled framework for characterizing the global propagation of cache reuse errors.
By directly minimizing the analytical upper bound in \cref{eq:optim_theoretical_ub}, we derive a cache reuse policy that accounts for the cumulative impact of errors across the entire trajectory.
As illustrated in \cref{fig:comparison_cache_policies}, the policy optimized via this theoretical bound (Orange) consistently achieves a lower reconstruction error $\|\bm x_{t_1}-\hat{\bm x}_{t_1}\|_1$ under a fixed computation budget ($K=10$) at the final timestep compared to the state-of-the-art ERTACache~\cite{DBLP:journals/corr25:ERTACache} (Blue).
This superiority stems from our bound's ability to capture the long-term dependency of errors, whereas prior methods often rely on local heuristics that fail to account for the exponential amplification of early-stage perturbations.

\subsection{Optimized Cache Reuse Policy}
\label{sec:method:method}

\textbf{From Theoretical Bounds to Empirical Alignment.}
While the policy derived from \cref{eq:optim_theoretical_ub} effectively minimizes the cumulative error, it exhibits a distinct behavior of prioritizing full computation in early denoising stages and clustering cache reuse toward the end.
To understand the underlying mechanics, we analyze the discrepancy between our analytical bound and the actual empirical error.
As illustrated in \cref{fig:estimated_error}, the analytical bound provides a conservative overestimation of the propagated error, particularly during the initial steps.
This gap stems from the fact that the theoretical bound assumes a worst-case error growth, which does not fully account for the intrinsic error-resilience and non-linear dynamics of diffusion models.
Motivated by this observation, we move beyond a conservative bound and introduce a bilevel optimization formulation.
This approach adaptively ``tightens'' the error estimation by learning a weighting function that aligns our theoretical framework with empirical generation quality, ultimately yielding a more refined and effective cache reuse policy.

\textbf{Bilevel Optimization Formulation.}
To address the gap of the analytical bound discussed above, we propose \textbf{G}lobal-Impact \textbf{Cache} (\textbf{GCache}).
The key idea is to move beyond a conservative upper bound by optimizing a parameterized error-weighting function that better aligns with empirical results.
Specifically, we parameterize the propagation exponent $w_{t}$ using the $d$-th degree Bernstein polynomials as follows:
\begin{equation}
    \label{eq:poly}
    w(t; \bm s) = \sum_{\nu = 0}^{d} \bm s_{\nu + 1} \binom{n}{\nu} t^\nu (1 - t)^{d - \nu},
\end{equation}
where $\bm s \in [s_\text{min}, s_\text{max} ]^{d+1}$ denotes a vector of learnable parameters (coefficients).
With this parameterization, the search for an optimal cache reuse policy can be formulated as a bilevel optimization problem:
\begin{align}
    \bm s
    & = \underset{\bm s \in [s_\text{min}, s_\text{max}]^{d+1}}{\arg \min} \, \mathcal{L}(\bm m^\star(\bm s)),
    \label{eq:bilevel_outer} \\
    \mathrm{s.t.} \quad \bm m^\star(\bm s)
    & = \underset{\bm m \in \mathcal{C}}{\arg \min} \sum_{i=1}^{N-1} \left\| \bm \epsilon^c_{t_{i+1}} \right\|_1 e^{w(t_{i+1}; \bm s)}.
    \label{eq:bilevel_inner}
\end{align}
Here, $\mathcal{C} = \{ \bm m: \| \bm m \|_0 =K \}$ denotes the feasible set for the policy $\bm m$ under a predefined budget constraint $K$.

Intuitively, the inner objective~\eqref{eq:bilevel_inner} seeks an optimal cache reuse policy $\bm m^\star$ that minimizes the current upper bound of the propagated error induced by the given polynomial parameters $\bm s$;
while the outer objective~\eqref{eq:bilevel_outer} optimizes the parameters $\bm s$ to minimize a generation quality metric $\mathcal{L}$ (e.g., LPIPS~\cite{DBLP:conf/cvpr18:LPIPS}), effectively tuning the error-weighting function to better reflect empirical generation quality.
However, standard first-order methods (e.g., SGD) are inapplicable to this problem, as neither objective is differentiable with respect to $\bm s$ or $\bm m$.
We therefore develop tailored strategies for both levels.

% \textbf{Solving the Bilevel Optimization Problem.}
% Directly optimizing the proposed bilevel problem using standard first-order methods (e.g., stochastic gradient descent) is infeasible, since neither the inner nor the outer objective is differentiable with respect to its parameters.
% Direct optimization of the proposed bilevel problem via standard first-order methods, such as Stochastic Gradient Descent (SGD), is not feasible because neither objective is directly differentiable with respect to its parameters.
% To circumvent this limitation, we develop distinct strategies tailored for
% To address this challenge, we design tailored solution strategies for the both \ding{182} inner and \ding{183} outer optimization problems.

\textbf{Inner Optimization via Dynamic Programming.}
The inner problem~\eqref{eq:bilevel_inner} can be cast as a constrained shortest path problem on a directed acyclic graph.
We define a cost matrix $E \in \mathbb{R}^{N \times N}$, where each entry
\begin{equation}
    E_{i,j} = \left\| \bm \delta_{t_i} - \bm \delta_{t_j} \right\|_1 e^{w(t_{j};\bm s)},
    \label{eq:error-matrix}
\end{equation}
quantifies the propagated error of reusing the residual from $t_i$ at step $t_j$.
Finding the optimal policy $\bm m^\star$ is equivalent to finding a path from $t_N$ to $t_1$ that minimizes cumulative cost with exactly $K$ refresh nodes.
This admits an efficient Dynamic Programming (DP) solution with $O(KN^2)$ complexity.
Given $N \le 100$ in modern schedulers, this overhead is negligible.

% \ding{182} The inner objective~\eqref{eq:bilevel_inner} can be reformulated as a classical constrained shortest path problem.
% Specifically, we define a propagated error matrix $E \in \mathbb{R}^{N \times N}$, where each entry,
% \begin{equation}
%     E_{i,j} = \left\| \bm \delta_{t_i} - \bm \delta_{t_j} \right\|_1 e^{w(t_{j};\bm s)},
%     \label{eq:error-matrix}
% \end{equation}
% represents the cost (propagated error) incurred by reusing the residual $\bm \delta_{t_i}$ at time $t_j$.
% Under this formulation, selecting an optimal cache reuse policy reduces to finding a path from $(N, N)$ to $(N-K+1, 1)$ that minimizes the cumulative propagation cost under a fixed refresh budget.
% In this graph interpretation, a direct computation step, moving from $(i, j)$ to $(i+1,j)$, incurs zero error, while reuse steps accumulate the costs defined in $E$.
% This problem is well-studied and admits an efficient dynamic programming solution with time complexity $O(KN^2)$ (see Appendix~\ref{app:dp_search} for details).
% Since the number of denoising steps $N$ is typically fewer than $100$, the computational overhead of this optimization is negligible in practice.

\textbf{Outer Optimization via Bayesian Optimization.}
Since evaluating the outer objective $\mathcal{L}(\bm m^\star(\bm s))$ requires a full inference pass over the training set, it is a high-cost black-box function.
We employ Bayesian Optimization (BO) to efficiently search the parameter space of  $\bm s$.
BO models the objective $f_\text{eval}(\bm s) = \mathcal{L}(\bm m^\star(\bm s)) $ using a Gaussian Process (GP) surrogate.
Let $\mathcal{D}_n = \{ (\bm s_i, y_i) \}_{i=1}^n$ be the history of $n$ evaluations.
The GP provides a posterior predictive distribution $p(y \mid \bm s, \bm S, \bm Y) \sim \mathcal{N}(\mu(\bm s), \sigma^2(\bm s))$, where the mean $\mu(\bm s)$ estimates performance and the variance $\sigma^2(\bm s)$ quantifies uncertainty:
\begin{align}
    \mu (\bm{s}) &= \bm{k}(\bm{s}, \bm{S}) \left( \bm{K} + \sigma_\epsilon^2 \mathbf{I} \right)^{-1} \bm{Y}, \label{eq:mu} \\
    \sigma^2(\bm{s}) &= k(\bm{s}, \bm{s}) - \bm{k}(\bm{s}, \bm{S}) \left( \bm{K} + \sigma_\epsilon^2 \mathbf{I} \right)^{-1} \bm{k}(\bm{s}, \bm{S})^\top. \label{eq:sigma}
\end{align}
Here, $\bm S = [\bm s_1, \bm s_2, \cdots, \bm s_n]^\top$ are reviously sampled parameters and $\bm Y = [y_1, y_2, \cdots, y_n]^\top$ are their corresponding empirical evaluations.
We use the Lower Confidence Bound (LCB)~\cite{DBLP:conf/icml10/UCB} as the acquisition function:
\begin{equation}
    A_{\text{LCB}}(\bm s; \mathcal{D}_n) = \mu(\bm s) - \kappa_n \sigma(\bm s),
\end{equation}
where $\kappa_n$ balances exploitation and exploration.
This framework allows GCache to iteratively discover the optimal propagation exponent that best aligns the theoretical bound with empirical performance.

Detailed implementation specifics for Dynamic Programming and Bayesian Optimization are deferred to Appendix~\ref{app:dp_search} and Appendix~\ref{app:bayesian-optm}, respectively.
Additionally, a thorough discussion regarding the selection of our optimization objective is provided in Appendix~\ref{app:opt_objective}.

\section{Experiments}
\label{sec:experiments}

\subsection{Experimental Settings}

We conduct experiments on four representative DiT-based diffusion models to validate the generality and effectiveness of our method across both video and image generation. 
Specifically, we evaluate three video diffusion backbones, Open-Sora 1.2~\cite{DBLP:journals/corr24:opensora}, CogVideoX~\cite{DBLP:conf/iclr25:CogVideoX}, and Wan 2.1~\cite{DBLP:journals/corr25:wan2.1}, as well as one strong text-to-image model, Flux-dev 1.0~\cite{flux2024}. Unless otherwise specified, all experiments are conducted on a single NVIDIA A800 80GB GPU, and results are reported by averaging over five random seeds. 
Additional experimental details are provided in Appendix~\ref{app:exp_details}, respectively.
% We evaluate GCache on four representative DiT-based diffusion models: three video backbones (Open-Sora~1.2~\cite{DBLP:journals/corr24:opensora}, CogVideoX~\cite{DBLP:conf/iclr25:CogVideoX}, and Wan~2.1~\cite{DBLP:journals/corr25:wan2.1}) and one image backbone (Flux-dev~1.0~\cite{flux2024}). Unless otherwise specified, all experiments are conducted on a single NVIDIA A800 80GB GPU and averaged over five random seeds. More details are provided in Appendix~\ref{app:exp_details}.

\textbf{Baselines.}
We compare against several recent diffusion acceleration baselines, including $\Delta$-DiT~\cite{DBLP:journals/corr24:Delta-DiT}, T-GATE~\cite{DBLP:journals/tmlr25:TGATE}, PAB~\cite{DBLP:conf/iclr25:PAB}, ProfilingDiT~\cite{icv25:ProfilingDiT}, FasterCache~\cite{DBLP:conf/iclr25:fastercache}, TeaCache~\cite{DBLP:conf/cvpr25:TeaCache}, and ERTACache~\cite{DBLP:journals/corr25:ERTACache}.

\textbf{Evaluation Metrics.}
Following prior work~\cite{DBLP:conf/cvpr25:TeaCache, DBLP:journals/corr25:ERTACache}, we evaluate video generation using the official 946 prompts provided by VBench~\cite{DBLP:conf/cvpr24:vbench}, and image generation using the official 30K prompts from COCO~\cite{DBLP:conf/eccv14:coco}. 
For efficiency, we measure end-to-end inference latency, from prompt ingestion to the generation of the final frame, and report speedup relative to the corresponding base model. For quality, we adopt four widely used metrics: VBench~\cite{DBLP:conf/cvpr24:vbench}, LPIPS~\cite{DBLP:conf/cvpr18:LPIPS}, PSNR, and SSIM. More details are provided in Appendix~\ref{app:evaluation-metrics}.

\subsection{Main Results}

\begin{table*}[!t]
\centering
\caption{
    Quantitative evaluation of efficiency and visual quality for video generation across different methods on three leading text-to-video diffusion models.
    \textbf{Bold} denotes the best performance under similar acceleration ratios.
    $\uparrow$ indicates higher is better, and $\downarrow$ indicates lower is better.
}
% \footnotesize
% \vspace{-0.2em}
\scriptsize
\label{tab:comparison}
\begin{tabular}{cccccccc}
    \toprule
    \multirow{2.5}{*}{\textbf{Model}} & \multirow{2.5}{*}{\textbf{Method}} & \multicolumn{2}{c}{\textbf{Efficiency}} & \multicolumn{4}{c}{\textbf{Visual Quality}} \\
    \cmidrule(lr){3-4} \cmidrule(lr){5-8}
    & & \textbf{Speedup}$\uparrow$ & \textbf{Latency (s)}$\downarrow$ & \textbf{VBench}$\uparrow$ & \textbf{LPIPS}$\downarrow$ & \textbf{SSIM}$\uparrow$ & \textbf{PSNR}$\uparrow$ \\
    
    \midrule

    \multirow{11}{*}{\makecell{\textbf{Open-Sora 1.2}\\(51 frames, 480P)}}
    % \multicolumn{7}{c}{\textbf{Open-Sora 1.2} (51 frames, 480P)} \\
    % \midrule
    & Open-Sora 1.2 ($T = 30$) & 1$\times$ & 44.56 & 79.22\% & - & - & - \\
    & $\Delta$-DiT~\cite{DBLP:journals/corr24:Delta-DiT} &  1.03$\times$ & - & 78.21\% & 0.5692 & 0.4811 & 11.91\\
    & T-GATE~\cite{DBLP:journals/tmlr25:TGATE} & 1.19$\times$ & - & 77.61\% & 0.3495 & 0.6760 & 15.50 \\
    & PAB-slow~\cite{DBLP:conf/iclr25:PAB} &  1.33$\times$ & 33.40 & 77.64\% & 0.1471 & 0.8405 & 24.50 \\
    & PAB-fast~\cite{DBLP:conf/iclr25:PAB} &  1.40$\times$ & 31.85 & 76.95\% & 0.1743 & 0.8220 & 23.58 \\
    & TeaCache-slow~\cite{DBLP:conf/cvpr25:TeaCache}  &  1.55$\times$ & 28.78 & 79.28\% & 0.1316 & 0.8415 & 23.62 \\
    & TeaCache-fast~\cite{DBLP:conf/cvpr25:TeaCache} &  2.25$\times$ & 19.84 & 78.48\% & 0.2511 & 0.7477 & 19.10 \\
    & ERTACache-slow~\cite{DBLP:journals/corr25:ERTACache}   & 1.55$\times$  & 28.75 & 79.36\% & 0.1006 & 0.8706 & 25.45 \\
    & ERTACache-fast~\cite{DBLP:journals/corr25:ERTACache}  & 2.47$\times$ & 18.04 & 78.64\% & 0.1659 & 0.8170 & 22.34 \\
    & \cellcolor[gray]{0.9}\textbf{GCache-slow} ($K = 18$) & \cellcolor[gray]{0.9}1.56$\times$ & \cellcolor[gray]{0.9}\textbf{28.54} & \cellcolor[gray]{0.9}\textbf{79.48\%} & \cellcolor[gray]{0.9}\textbf{0.0509} & \cellcolor[gray]{0.9}\textbf{0.9247} & \cellcolor[gray]{0.9}\textbf{31.52} \\
    & \cellcolor[gray]{0.9}\textbf{GCache-fast} ($K = 11$)& \cellcolor[gray]{0.9}2.54$\times$ & \cellcolor[gray]{0.9}\textbf{17.48} & \cellcolor[gray]{0.9}78.44\% & \cellcolor[gray]{0.9}\textbf{0.1363} & \cellcolor[gray]{0.9}\textbf{0.8428} & \cellcolor[gray]{0.9}\textbf{24.92} \\
    
    \midrule

    \multirow{9}{*}{\makecell{\textbf{CogVideoX-2B}\\(48 frames, 480P)}}
    % \multicolumn{7}{c}{\textbf{CogVideoX-2B} (48 frames, 480P)} \\
    % \midrule
    & CogVideoX-2B ($T = 50$) & 1$\times$ & 78.48 & 80.18\% & - & - & - \\
    & $\Delta$-DiT~\cite{DBLP:journals/corr24:Delta-DiT} & 1.26$\times$ & 62.50 & 79.09\% & 0.4053 & 0.6126 & 16.15 \\
    & PAB~\cite{DBLP:conf/iclr25:PAB} & 1.35$\times$ & 57.98 & 79.76\% & 0.0860 & 0.8978 & 28.04 \\
    & FasterCache~\cite{DBLP:conf/iclr25:fastercache} & 1.62$\times$ & 48.44 & 79.83\% & 0.0766 & 0.9066 & 28.93 \\
    & TeaCache~\cite{DBLP:conf/cvpr25:TeaCache} & 2.92$\times$ & 26.88 & 79.00\% & 0.2057 & 0.7614 & 20.97 \\
    & ERTACache-slow~\cite{DBLP:journals/corr25:ERTACache} & 1.62$\times$ & 48.44 & 79.30\% & 0.0368 & 0.9394 & 32.77 \\
    & ERTACache-fast~\cite{DBLP:journals/corr25:ERTACache}& 2.93$\times$ & 26.78 & 78.79\% & 0.1012 & 0.8702 & 26.44 \\
    & \cellcolor[gray]{0.9}\textbf{GCache-slow} ($K = 31$) & \cellcolor[gray]{0.9}1.62$\times$ & \cellcolor[gray]{0.9}\textbf{48.40} & \cellcolor[gray]{0.9}\textbf{79.36\%} & \cellcolor[gray]{0.9}\textbf{0.0178} & \cellcolor[gray]{0.9}\textbf{0.9647} & \cellcolor[gray]{0.9}\textbf{37.53} \\
    & \cellcolor[gray]{0.9}\textbf{GCache-fast} ($K = 17$) & \cellcolor[gray]{0.9}2.93$\times$ & \cellcolor[gray]{0.9}\textbf{26.76} & \cellcolor[gray]{0.9}78.30\% & \cellcolor[gray]{0.9}\textbf{0.0721} & \cellcolor[gray]{0.9}\textbf{0.9042} & \cellcolor[gray]{0.9}\textbf{29.14} \\
    
    \midrule

    \multirow{6}{*}{\makecell{\textbf{Wan2.1-1.3B}\\(81 frames, 480P)}}
    % \multicolumn{7}{c}{\textbf{Wan2.1-1.3B} (81 frames, 480P)} \\
    % \midrule
    & Wan2.1-1.3B ($T = 50$) & 1$\times$ & 199 & 81.30\% & - & - & - \\
    & TeaCache~\cite{DBLP:conf/cvpr25:TeaCache} & 2.00$\times$ & 99.5 & 76.04\% & 0.2913 & 0.5685 & 16.17 \\
    & ProfilingDiT~\cite{icv25:ProfilingDiT} & 2.01$\times$ & 99 & 76.15\% & 0.1256 & 0.7899 & 22.02 \\
    & ERTACache~\cite{DBLP:journals/corr25:ERTACache} & 2.17$\times$ & 91.7 & 80.73\% & 0.1095 & 0.8200 & 23.77 \\
    & \cellcolor[gray]{0.9}\textbf{GCache-slow} ($K = 24$) & \cellcolor[gray]{0.9}2.17$\times$ & \cellcolor[gray]{0.9}\textbf{91.6} & \cellcolor[gray]{0.9}\textbf{80.81\%} & \cellcolor[gray]{0.9}\textbf{0.0316} & \cellcolor[gray]{0.9}\textbf{0.9475} & \cellcolor[gray]{0.9}\textbf{32.44} \\
    & \cellcolor[gray]{0.9}\textbf{GCache-fast} ($K = 16$) & \cellcolor[gray]{0.9}3.01$\times$ & \cellcolor[gray]{0.9}\textbf{66.1} & \cellcolor[gray]{0.9}\textbf{79.86\%} & \cellcolor[gray]{0.9}\textbf{0.0828} & \cellcolor[gray]{0.9}\textbf{0.8854} & \cellcolor[gray]{0.9}\textbf{22.06} \\
    
    \bottomrule 
    
\end{tabular}
\vskip -0.04in
\end{table*}

\begin{table*}[!t]
% \vskip 0.1in
\centering
% \setlength{\tabcolsep}{2pt}
% \footnotesize
\scriptsize
\caption{Quantitative evaluation of efficiency and visual quality for image generation across different methods on Flux-dev 1.0.
\textbf{Bold} denotes the best performance under similar acceleration ratios.
$\uparrow$ indicates higher is better, and $\downarrow$ indicates lower is better.}
\label{tab:comparison_Flux}
\begin{tabular}{lccccc}
\toprule
\multirow{2.5}{*}{\textbf{Method}} & \multicolumn{2}{c}{\textbf{Efficiency}} & \multicolumn{3}{c}{\textbf{Visual Quality}} \\
\cmidrule(lr){2-3} \cmidrule(lr){4-6}
& \textbf{Speedup}$\uparrow$ & \textbf{Latency (s)}$\downarrow$ & \textbf{LPIPS}$\downarrow$ & \textbf{SSIM} $\uparrow$ & \textbf{PSNR} $\uparrow$ \\
\midrule
% \multicolumn{7}{c}{\textbf{Open-Sora 1.2} (51 frames, 480P)} \\
% \midrule
Flux-dev 1.0 ($T = 30$) & 1$\times$ & 15.96 & - & - & - \\

TeaCache~\cite{DBLP:conf/cvpr25:TeaCache}  &  2.84$\times$ & 5.62 & 0.4427 & 0.7445 & 16.48 \\

ERTACache~\cite{DBLP:journals/corr25:ERTACache} & 2.87$\times$ & 5.56 & 0.2658 & 0.7863 & 20.60 \\

\rowcolor[gray]{0.9}
\textbf{GCache-slow} ($K=14$) & 2.05$\times$ & 7.78 & \textbf{0.0900} & \textbf{0.9129} & \textbf{28.50} \\

\rowcolor[gray]{0.9}
\textbf{GCache-fast} ($K=10$) & 2.87$\times$ & \textbf{5.56} & \textbf{0.1825} & \textbf{0.8423} & \textbf{23.76} \\

\bottomrule
\end{tabular}
\vspace{-0.2in}
\end{table*}

\begin{figure}[t]
    \centering
    % \vskip 0.1in
    \includegraphics[width=\columnwidth]{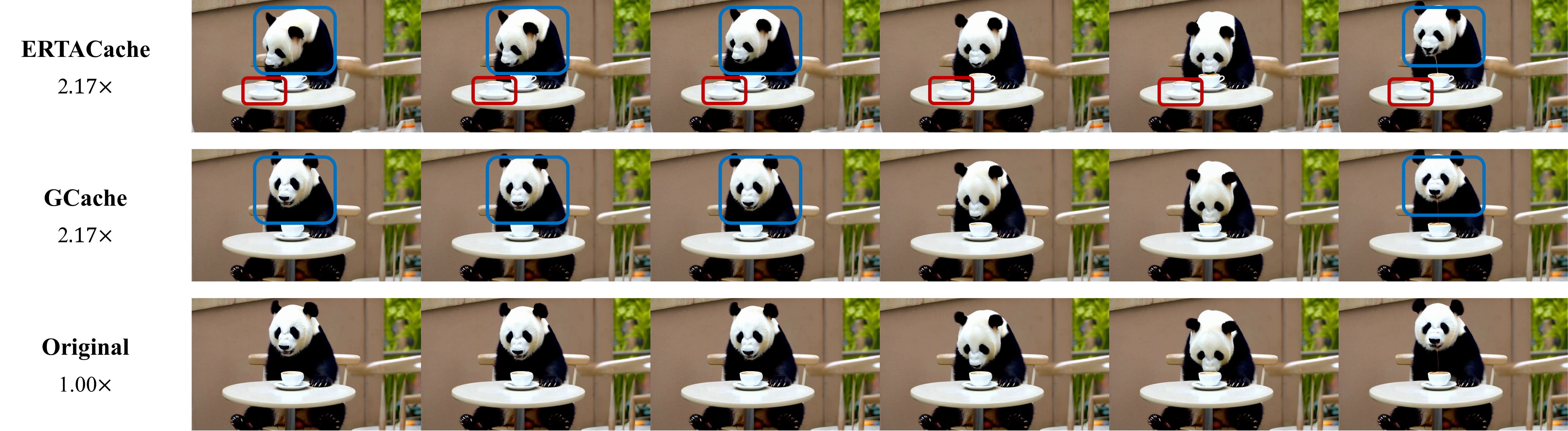}
    % \vskip -0.1in
    \caption{
        Qualitative comparison of video generation results on \textbf{Wan~2.1}.
        GCache consistently preserves higher alignment with the original, while ERTACache exhibits noticeable object and motion misalignment.
        \textcolor{red}{Red boxes} highlight object misalignment, and \textcolor{blue}{blue boxes} indicate motion misalignment.
        Best viewed when zoomed in.
    }
    \label{fig:q_comparison_video}
    \vskip -0.15in
\end{figure}

\textbf{Quantitative Evaluation.}
Across all evaluated settings, GCache consistently achieves state-of-the-art performance, delivering both faster inference and higher generation quality than existing cache-based acceleration methods.
As shown in \cref{tab:comparison}, GCache-Slow consistently outperforms the strongest baseline, ERTACache, across all video backbones, achieving higher VBench scores and reducing LPIPS by over $50\%$ on average while maintaining comparable acceleration.
Notably, on Wan~2.1, GCache-Slow reduces LPIPS from $0.1095$ to $0.0316$ under the same $2.17\times$ speedup, while GCache-Fast further achieves a $3.01\times$ speedup and still surpasses ERTACache in generation quality ($0.0828$ vs.\ $0.1095$ LPIPS).
Similarly, on Flux-dev~1.0 (\cref{tab:comparison_Flux}), GCache achieves an LPIPS of $0.1825$ under a $2.87\times$ speedup setting, significantly outperforming prior methods.
These results highlight a key distinction between GCache and prior cache-based methods.
While methods such as ERTACache explicitly introduce additional error estimation and rectification mechanisms to compensate for reuse-induced approximation errors, GCache focuses directly on optimizing the cache reuse policy itself from a global-impact perspective.
By allocating computation to the timesteps that most influence final generation quality, GCache achieves superior speed--quality trade-offs without introducing any additional inference-time computation.

% \begin{figure}[t]
%     \centering
%     % \vskip 0.1in
%     \includegraphics[width=0.9\columnwidth]{figures/q_comparison_img.pdf}
%     \vskip -0.1in
%     \caption{Qualitative comparison of image generation results on Flux-dev 1.0.
%     Best viewed when zoomed in.}
%     \label{fig:q_comparison_img}
%     \vskip -0.1in
% \end{figure}

\textbf{Qualitative Comparison.}
\cref{fig:q_comparison_video} shows video generation results on WAN~2.1.
Obviously, ERTACache suffers from motion and object misalignment under aggressive cache reuse, including misaligned motion dynamics (e.g., panda face motion) and object-level deviations (e.g., items on the table) from the original outputs.
In contrast, GCache preserves consistent motion patterns and object semantics.
Similarly, Figure~\ref{fig:q_comparison_img} presents image generation results on Flux-dev 1.0.
ERTACache exhibits clear semantic and spatial misalignment under the same prompts, such as generating four smoking stacks when the prompt specifies \emph{two}, as well as malformed human structures and incorrect object relationships.
By contrast, GCache maintains semantic correctness and spatial coherence, closely matching the original model outputs.
Additional results and discussions are provided in Appendix~\ref{app:qual_comparison}.

\begin{table}[tbp]
    \centering
    % 第一个表格的 minipage
    \begin{minipage}{0.45\textwidth}
        \centering
        \caption{Ablation study on polynomial order $d$.}
        \label{experiment: ablation study poly order}
        \footnotesize
        \vskip 0.05in
        \begin{tabular}{cccc}
        \toprule
        $d$ & \textbf{LPIPS}$\downarrow$ & \textbf{SSIM}$\uparrow$ & \textbf{PSNR}$\uparrow$ \\
        \midrule
        $1$ & 0.1114 & 0.8591 & 26.15 \\
        $2$ & 0.0733 & 0.9042 & 29.09 \\
        $3$ & \textbf{0.0721} & \textbf{0.9042} & \textbf{29.14} \\
        $4$ & 0.0733 & 0.9016 & 28.84 \\
        \bottomrule
    \end{tabular}
    \end{minipage}
    \hfill % 在两个 minipage 之间添加弹性间距，将它们推向两端
    % 第二个表格的 minipage
    \begin{minipage}{0.45\textwidth}
        \centering
        \caption{
            Ablation study on the outer objective $\mathcal{L}$.
            % The hybrid objective (LPIPS+SSIM) achieves the optimal balance between perceptual quality and structural fidelity.
        }
        \label{tab:objective_comparison}
        \footnotesize
        \vskip 0.05in
        \begin{tabular}{cccc}
        \toprule
        \textbf{Objective} & \textbf{LPIPS $\downarrow$} & \textbf{SSIM $\uparrow$} & \textbf{PSNR $\uparrow$} \\ 
        \midrule
        LPIPS       & 0.0733 & 0.9016 & 28.84 \\
        SSIM        & 0.0736 & \textbf{0.9045} & 29.10 \\
        LPIPS+SSIM  & \textbf{0.0721} & 0.9042 & \textbf{29.14} \\ 
        \bottomrule
    \end{tabular}
    \end{minipage}
    \vskip -0.2in
\end{table}

\begin{wrapfigure}{r}{0.35\linewidth}
    \centering
    \vskip -0.15in
    \includegraphics[width=\linewidth]{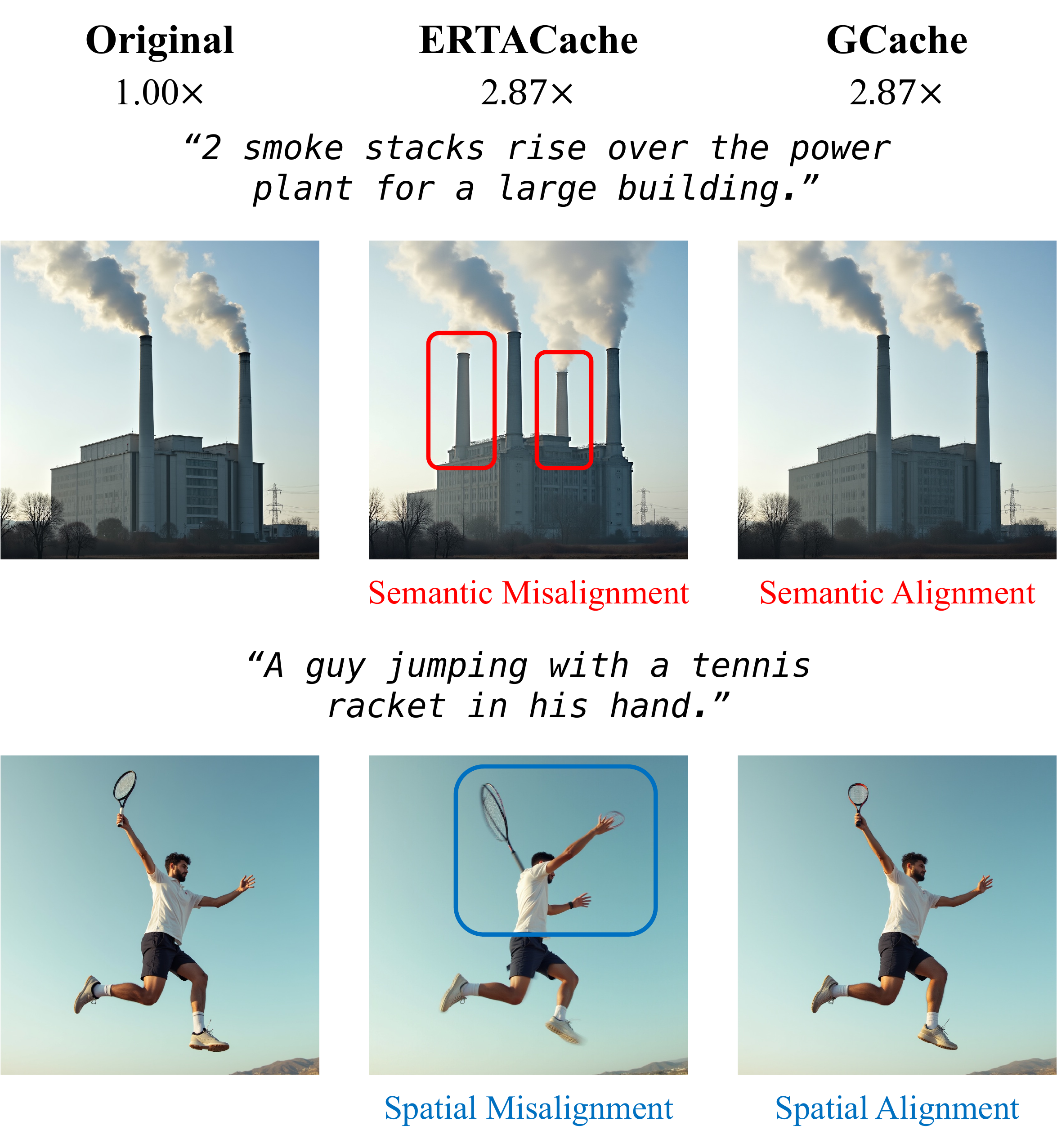}
    \vskip -0.05in
    \caption{Qualitative comparison of image generation results on Flux-dev 1.0.
    Best viewed when zoomed in.}
    \label{fig:q_comparison_img}
    \vskip -0.2in
\end{wrapfigure}

\subsection{Ablation Studies and Empirical Analysis}
\label{main:ablation_study}
% \label{app:exp_results:poly_degree}

% \begin{table}[tbp]
%     % \vskip -0.2in
%     \centering
%     \caption{
%         Ablation study on polynomial order $d$.
%         $\uparrow$: higher is better; $\downarrow$: lower is better.
%     }
%     \label{experiment: ablation study poly order}
%     \footnotesize
%     \vskip 0.1in
%     \begin{tabular}{cccc}
%         \toprule
%         $d$ & \textbf{LPIPS}$\downarrow$ & \textbf{SSIM}$\uparrow$ & \textbf{PSNR}$\uparrow$ \\
%         \midrule
%         $1$ & 0.1114 & 0.8591 & 26.15 \\
%         $2$ & 0.0733 & 0.9042 & 29.09 \\
%         $3$ & \textbf{0.0721} & \textbf{0.9042} & \textbf{29.14} \\
%         $4$ & 0.0733 & 0.9016 & 28.84 \\
%         \bottomrule
%     \end{tabular}
% \end{table}

\textbf{Ablation Study on Poly Degree $d$.}
We study the effect of the polynomial degree $d$ used in Eq.~\ref{eq:poly} for modeling the propagation exponent $w_t$.
Specifically, we evaluate different choices of $d$ under a fixed budget $K=17$ on CogVideoX-2B.
As shown in Table~\ref{experiment: ablation study poly order}, the best performance is achieved at $d=3$, yielding the lowest LPIPS and the highest SSIM and PSNR.
Therefore, we set $d=3$ as the default choice in all experiments.

% \subsection{Impact of the Outer Objective}
% \label{app:exp_results:outer_obj}

% \begin{table}[tbp]
%     \centering
%     \caption{
%         Ablation study on the outer objective formulation.
%         The hybrid objective (LPIPS+SSIM) achieves the optimal balance between perceptual quality and structural fidelity.
%     }
%     \label{tab:objective_comparison}
%     \footnotesize
%     \vskip 0.1in
%     \begin{tabular}{cccc}
%         \toprule
%         \textbf{Objective} & \textbf{LPIPS $\downarrow$} & \textbf{SSIM $\uparrow$} & \textbf{PSNR $\uparrow$} \\ 
%         \midrule
%         LPIPS       & 0.0733 & 0.9016 & 28.84 \\
%         SSIM        & 0.0736 & \textbf{0.9045} & 29.10 \\
%         LPIPS+SSIM  & \textbf{0.0721} & 0.9042 & \textbf{29.14} \\ 
%         \bottomrule
%     \end{tabular}
% \end{table}

\textbf{Impact of the Outer Objective $\mathcal{L}$.}
To investigate the sensitivity of GCache to the choice of the outer objective, we evaluate three loss formulations on CogVideoX-2B ($K=17$): LPIPS, SSIM, and a hybrid LPIPS+SSIM objective.
As summarized in \cref{tab:objective_comparison}, while individual losses focus on specific image attributes (LPIPS on perceptual features and SSIM on structural integrity), their combination (LPIPS+SSIM) yields the best overall performance across all metrics.
Specifically, the hybrid objective achieves the lowest LPIPS (0.0721) and the highest PSNR (29.14), suggesting that a multi-faceted supervision signal is crucial for optimizing the reuse policy.

% \subsection{Additional Experimental Results}
% \label{main_paper_ablation_study}

% We conduct extensive ablation studies and empirical analyses to validate the design choices of GCache and better understand its underlying mechanisms. Specifically, we study polynomial degree selection (Appendix~\ref{app:exp_results:poly_degree}), outer objective design (Appendix~\ref{app:exp_results:outer_obj}), robustness to prompt distribution shifts (Appendix~\ref{app:exp_results:prompt_distribution}), generalization across resolutions (Appendix~\ref{app:exp_results:Generalization_resolution}), the effectiveness of the learned policy (Appendix~\ref{app:exp_results:effect_policy}), and the fidelity of the pre-computed local error proxy (Appendix~\ref{app:exp_results:precompute_error}).

\textbf{Additional Experimental Results.}
We provide more comprehensive evaluations and empirical studies in Appendix~\ref{app:exp_results}.
These include:
% (i) an analysis of hyperparameter sensitivity, such as the polynomial degree $d$ (Appendix~\ref{app:exp_results:poly_degree});
% (ii) an investigation into the impact of different outer objectives (Appendix~\ref{app:exp_results:outer_obj});
(i) extensive robustness tests across diverse prompt distributions and spatial resolutions (Appendix~\ref{app:exp_results:prompt_distribution} and \ref{app:exp_results:Generalization_resolution});
and (ii) a detailed validation of the learned policy's effectiveness and the fidelity of our pre-computed error proxy (Appendix~\ref{app:exp_results:effect_policy} and \ref{app:exp_results:precompute_error}).
These results further substantiate the stability and generalization of GCache across various scenarios.

\section{Conclusion}

In this paper, we first identify the misalignment between local reuse discrepancies and global generation error, establishing a formal theoretical characterization of error propagation dynamics in cache-based acceleration.
We demonstrate that policies optimized strictly via conservative analytical bounds are often sub-optimal in practice.
To bridge this gap, we introduce \textbf{G}lobal-Impact \textbf{Cache} (\textbf{GCache}), a framework that reformulates the policy search as a bilevel optimization problem.
This approach effectively reconciles theoretical error control with empirical perceptual quality.
Extensive evaluations across various image and video backbones show that GCache consistently outperforms prior caching strategies.
% , achieving significant improvements in generation quality while maintaining substantial inference speedups.

% \clearpage
% \newpage

{
\bibliography{ref}
}

\clearpage
\newpage

%%%%%%%%%%%%%%%%%%%%%%%%%%%%%%%%%%%%%%%%%%%%%%%%%%%%%%%%%%%%

\appendix

\bigskip
\begin{center}
{\Large\bf Appendix}
\end{center}
% \vskip -0.1in

% \noindent\textbf{Roadmap.}
% We first provide implementation details in Appendix~\ref{app:imp_details}.
% We then present the missing proofs of the theorems in Appendix~\ref{app:proofs}.
% Next, we provide the experimental details and additional experimental results in Appendix~\ref{app:exp_details} and Appendix~\ref{app:exp_results}, respectively.
% Finally, we discuss the related work in Appendix~\ref{app:rel_work}.
% \YXC{TODO: quantitative comparison as a standalone section.}

% \begin{appendices}

% \addtocontents{toc}{\protect\setcounter{tocdepth}{-1}} % 让主目录停止记录附录后的详细章节
    
% \section*{Contents} % 附录目录的标题
% \etocsettocdepth{subsection} % 设置附录目录显示的深度
\addtocontents{toc}{\protect\setcounter{tocdepth}{2}}
% \setcounter{tocdepth}{4}
% 临时把目录链接颜色改成 citecolor
\hypersetup{
    linkcolor=citecolor
}

\tableofcontents

% 恢复正文原来的 linkcolor
\hypersetup{
    linkcolor=linkcolor
}

% \noindent
% \textbf{Appendix Overview.}  
% This appendix is organized as follows:

% \begin{itemize}
%     \item Appendix~\ref{app:imp_details} presents the implementation details of \textsc{GCache}, including the overall optimization procedure and the dynamic programming solver for cache policy search.
    
%     \item Appendix~\ref{app:proofs} provides the complete proofs of the theoretical results introduced in the main paper.

%     \item Appendix~\ref{app:rel_work} discusses related work and broader connections to prior diffusion acceleration methods.
    
%     \item Appendix~\ref{app:exp_details} provides supplementary experimental details, including dataset settings, evaluation protocols, and implementation configurations.
    
%     \item Appendix~\ref{app:exp_results} presents additional experimental results, including extended qualitative and quantitative evaluations.

%     \item Appendix~\ref{app:qual_comparison} provides additional qualitative comparisons, including image generation fidelity, cross-backbone video generation under shared prompts, and temporal consistency analysis.
    
% \end{itemize}

% \begin{appendices}

\clearpage
\section{Implementation Details}
\label{app:imp_details}

\subsection{Inner Optimization via Dynamic Programming}
\label{app:dp_search}

\begin{algorithm}[h]
\caption{Dynamic Programming for Cache Reuse Mask Search}
\label{alg:dp-mask}
\begin{algorithmic}[1]
\fontsize{9.4pt}{11pt}\selectfont

\REQUIRE Error matrix $E \in \mathbb{R}^{N\times N}$, cache refresh budget $K$
\ENSURE Reuse mask $m \in \{0,1\}^N$ with $\|m\|_0 = K$

\STATE Define the cumulative propagated error:
\(
\mathcal{E}_{j,i} = \sum_{\tau=j+1}^{i-1} E_{j,\tau},
\quad 0 \le j < i \le N
\)
\STATE Initialize DP table $dp[k,i]$ with $+\infty$ and parent table $par[k,i]$ with $-1$ for all $k,i$
\STATE $dp[1,0] \leftarrow 0$ 
% \hfill \textit{// Force refresh at the first step}

\FOR{$k = 2, \ldots, K$}
  \FOR{$i = k-1, \ldots, N-1$}
    \STATE $dp[k,i] \gets
      \min\limits_{0 \le j < i}
      \big(dp[k-1,j] + \mathcal{E}_{j,i}\big)$
    \STATE $par[k,i] \gets
      \arg\min\limits_{0 \le j < i}
      \big(dp[k-1,j] + \mathcal{E}_{j,i}\big)$
  \ENDFOR
\ENDFOR

\STATE $r_K \gets
  \arg\min\limits_{0 \le j \le N-1}
  \big(dp[K,j] + \mathcal{E}_{j,N}\big)$
\STATE Construct $m$ from $\{r_1,\dots,r_K\}$ by backtracking using $par$
\STATE \textbf{return} $m$
\end{algorithmic}
\end{algorithm}

Given the propagated error matrix $E$ defined in Eq.~\ref{eq:error-matrix}, where each entry encodes the cost incurred by reusing a cached residual across timesteps, selecting an optimal cache reuse policy amounts to minimizing the cumulative propagation error under a fixed refresh budget.

Recall that a cache reuse policy can be characterized by a set of $K$ refresh timesteps, which partition the denoising trajectory into $K$ segments.
Within each segment, intermediate timesteps reuse the most recent cached residual and accumulate the corresponding propagation costs.
Consequently, the total error induced by a policy is the sum of propagation costs over all segments, which naturally leads to a shortest-path formulation.

To solve this problem efficiently, we adopt a dynamic programming approach.
We define the DP state $dp[k,i]$ as the minimum accumulated propagation error when the $k$-th cache refresh occurs at timestep $i$.
Transitioning from a previous refresh point $j < i$ to $i$ incurs an additional cost given by the propagated error $\mathcal{E}_{j,i}$, which aggregates reuse costs between these two refresh points.
The objective is then to find $K$ refresh points that minimize the total accumulated error while satisfying the constraint $\|m\|_0 = K$.

Algorithm~\ref{alg:dp-mask} summarizes the resulting dynamic programming procedure, which runs in $O(KN^2)$ time.
Since the number of denoising steps $N$ is typically fewer than $100$ in practice, the computational overhead of this optimization is totally negligible.

\subsection{Outer Optimization via Bayesian Optimization}
\label{app:bayesian-optm}

\begin{algorithm}[h]
   \caption{Bilevel Optimization for Cache Policy Search}
   \label{alg:overall_algorithm_app}
\begin{algorithmic}[1]
\fontsize{9.4pt}{11pt}\selectfont
   \REQUIRE Initial observations $\mathcal{D}_0 = \{(\bm{s}_i, y_i)\}_{i=1}^m$, exploration parameter $\kappa_n$, total optimization step $T$
   % \STATE {\bfseries Output:} Optimal coefficients $\bm{s}^\star$ and corresponding policy $\bm{m}^\star$.
   
   \FOR{$n = 0$ {\bfseries to} $T-1$}
      % \STATE // \textit{Outer Level: Bayesian Optimization}
      \STATE Update GP surrogate using $\mathcal{D}_n$
      % \STATE Solve for candidate:\\
      \STATE $\bm{s}_{n+1} = \underset{\bm s}{\arg \min} \, \mu(\bm{s}) - \kappa_n \sigma(\bm{s}) $
      
      % \STATE // \textit{Inner Level: Solve for Optimal Policy $\bm{m}^\star$}
      % \STATE Construct error matrix $E \in \mathbb{R}^{N \times N}$ with $E_{i,j} = \| \bm{\delta}_{t_i} - \bm{\delta}_{t_j} \|_1 e^{w(t_j; \bm{s}_{n+1})}$.
      % \STATE Find $\bm{m}^\star(\bm{s}_{n+1})$ by solving the Constrained Shortest Path via DP:\\
      % $\bm{m}^\star(\bm{s}_{n+1}) = \arg \min_{\bm{m}} \sum \text{costs in } E \text{ s.t. } K \text{ steps}$.
      \STATE $\bm m^\star(\bm s_{n+1}) = \underset{\bm m \in \mathcal{C}}{\arg \min} \sum_{i=1}^{N-1} \left\| \bm \epsilon^c_{t_{i+1}} \right\|_1 e^{w(t_{i+1}; \bm s_{n+1})}$
      
      % \STATE // \textit{Evaluation}
      % \STATE Execute full denoising and compute loss:\\
      \STATE $y_{n+1} = \mathcal{L}(\bm{m}^\star(\bm{s}_{n+1}))$
      
      % \STATE // \textit{Update}
      \STATE $\mathcal{D}_{n+1} = \mathcal{D}_n \cup \{ (\bm{s}_{n+1}, y_{n+1}) \}$
   \ENDFOR
   \STATE Update GP surrogate using $\mathcal{D}_T$
   \STATE $\bm s^\star = \underset{\bm s}{\arg \min} \, \mu(\bm s)$
   \STATE $\bm m^\star = \underset{\bm m \in \mathcal{C}}{\arg \min} \sum_{i=1}^{N-1} \left\| \bm \epsilon^c_{t_{i+1}} \right\|_1 e^{w(t_{i+1}; \bm s^\star)}$
\end{algorithmic}
\end{algorithm}

The comprehensive optimization procedure for GCache is summarized in Algorithm~\ref{alg:overall_algorithm_app}.
At each iteration $n+1$, we maintain a Gaussian Process (GP) surrogate to model the objective landscape over the parameter space $\bm s$.
The next candidate is identified by minimizing the Lower Confidence Bound (LCB) acquisition function, which strategically balances exploration and exploitation:
\begin{equation}
    \bm s_{n+1} = \underset{\bm s}{\arg \min} \, A_{\text{LCB}}(\bm s; \mathcal{D}_n).
\end{equation}
Conditioned on the proposed parameters $\bm s_{n+1}$, the inner optimization level formulates the weighted error objective and identifies the optimal cache reuse policy $\bm{m}^\star(\bm{s}_{n+1})$.
This is achieved by minimizing the accumulated propagation error subject to the budget constraint $\bm m \in \mathcal{C}$(\cref{eq:optim_theoretical_ub}), leveraging the efficient dynamic programming solver detailed in Appendix~\ref{app:dp_search}.
Subsequently, we evaluate the empirical generation loss $y_{n+1} = f_\text{eval}(\bm s_{n+1}) = \mathcal{L}(\bm m^\star(\bm s_{n+1}))$ to augment the observation set $\mathcal{D}_{n+1} = \mathcal{D} \cup \{ (\bm s_{n+1}, y_{n+1}) \}$ and refine the GP surrogate.
This iterative cycle continues until the evaluation budget is exhausted, ultimately yielding an optimized propagation exponent that aligns the analytical error bound with empirical generation quality.

\subsection{Optimization Objective}
\label{app:opt_objective}

In the outer objective of the bilevel optimization framework (\cref{eq:bilevel_outer}), we define the loss function $\mathcal{L}$ to evaluate the empirical generation quality.
Specifically, we utilize a joint objective that combines the Learned Perceptual Image Patch Similarity (LPIPS)~\cite{DBLP:conf/cvpr18:LPIPS} and the Structural Similarity Index Measure (SSIM):
\begin{equation}
\mathcal{L} = \mathcal{L}_{\text{LPIPS}} + (1 - \mathcal{L}_{\text{SSIM}}).
\end{equation}
The rationale for this dual-component objective is twofold:
\begin{itemize}
    [leftmargin=1.2em, topsep=-0.5pt, itemsep=3pt, parsep=3pt]
    \item \textbf{Semantic and Perceptual Fidelity}:
    LPIPS leverages deep features from pre-trained networks (e.g., VGG or AlexNet) to quantify high-level perceptual similarity.
    In the context of cache-based acceleration, aggressive reuse can lead to ``semantic drift'' or the loss of fine-grained textures that traditional metrics often fail to capture.
    LPIPS ensures that the optimized caching policy preserves the overall visual intent and semantic coherence of the generated samples.
    \item \textbf{Structural and Pixel-level Integrity}:
    While LPIPS is effective for high-level perception, it can occasionally overlook local structural misalignments.
    By incorporating $1 - \mathcal{L}_\text{SSIM}$, we explicitly penalize deviations in local luminance, contrast, and spatial structure.
    This term acts as a regularizer to ensure that the accelerated trajectory remains spatially faithful to the unperturbed baseline, preventing artifacts such as ghosting or structural blurring that may arise during the accumulation of propagation errors.
\end{itemize}
To empirically justify the rationale behind our hybrid optimization objective, we conduct an ablation study comparing various loss formulations, with detailed results provided in~\cref{main:ablation_study}.
The findings corroborate that the combination of LPIPS and SSIM achieves a superior balance between perceptual fidelity and structural consistency.

\section{Proofs}
\label{app:proofs}

\subsection{Proof for \cref{theorem:error-bound_single-step_PF-ODE}}

\begin{lemma}[Single-Step Propagation Error]
    \label{lemma: single-step update error}
    Under Assumption~\ref{assump:regularity}, let $\bm{x}_{t}$ denote the state at time $t$ along the ground-truth ODE trajectory, and let $\hat{\bm{x}}_{t}$ be the state generated by the Euler solver using the learned velocity $\bm{v}_\theta$.
    For a backward step size $h_{t_{i+1}} = t_{i+1} - t_i > 0$, the error at time $t_i$ is bounded by:
    \begin{equation}
        \left\| \bm x_{t_i} - \hat{\bm x}_{t_i} \right\|_1
        \le \left(1 + h_{t_{i+1}} L_{t_{i+1}}\right) \left\| \bm x_{t_{i+1}} - \hat{\bm x}_{t_{i+1}} \right\|_1
        + h_{t_{i+1}} \eta
        + \frac{h_{t_{i+1}}^2}{2} M.
    \end{equation}
\end{lemma}

\begin{proof}
    By Taylor's theorem, the ground-truth state $\bm x_{t_i}$ can be expanded about $t_{i+1}$ by:
    \begin{align}
        \bm x_{t_i}
        & = \bm x_{t_{i+1}}
        - (t_{i+1} - t_i) \left. \frac{\mathrm{d}\bm x_t}{\mathrm{d}t} \right|_{t_{i+1}}
        + \left.\frac{(t_{i+1} - t_i)^2}{2} \frac{\mathrm{d}^2\bm x_t}{\mathrm{d}t^2} \right|_{\xi_{i+1}} \notag \\
        & = \bm x_{t_{i+1}}
        - h_{t_{i+1}} \bm v(\bm x_{t_{i+1}}, t_{i+1})
        + \frac{h_{t_{i+1}}^2}{2} \left.\frac{\mathrm{d}}{\mathrm{d}t} \bm v(\bm x_t, t) \right|_{\xi_{i+1}},
    \end{align}
    where $\xi_{i+1} \in [t_i, t_{i+1}]$ is an intermediate time point arising from the Lagrange form of the remainder, and we use the relation $\frac{\mathrm{d}\bm x_t}{\mathrm{d}t} = \bm v(\bm x_t, t)$.

    The Euler update for $\hat{\bm x}_{t_i}$ is defined as:
    \begin{equation}
        \hat{\bm x}_{t_i}
        = \hat{\bm x}_{t_{i+1}}
        - h_{t_{i+1}} \bm v_\theta(\hat{\bm x}_{t_{i+1}}, t_{i+1}).
    \end{equation}
    Subtracting the discrete update from the continuous Taylor expansion, we obtain:
    \begin{align}
        \bm x_{t_i} - \hat{\bm x}_{t_i}
        = \bm x_{t_{i+1}} - \hat{\bm x}_{t_{i+1}}
        - h_{t_{i+1}} \left( \bm v(\bm x_{t_{i+1}}, t_{i+1}) - \bm v_\theta(\hat{\bm x}_{t_{i+1}}, t_{i+1}) \right)
        + \frac{h_{t_{i+1}}^2}{2} \left. \frac{\mathrm{d}}{\mathrm{d}t} \bm v(\bm x_t, t) \right|_{\xi_{i+1}}.
    \end{align}
    To bound the difference in velocity terms, we decompose it into a propagation component and an approximation component:
    \begin{align}
        \bm v(\bm x_{t_{i+1}}, t_{i+1}) - \bm v_\theta(\hat{\bm x}_{t_{i+1}}, t_{i+1})
        & = \underbrace{\bm v(\bm x_{t_{i+1}}, t_{i+1}) - \bm v_\theta(\bm x_{t_{i+1}}, t_{i+1})}_{\text{Approximation Error}} \notag \\
        & \quad \ + \underbrace{\bm v_\theta(\bm x_{t_{i+1}}, t_{i+1}) - \bm v_\theta(\hat{\bm x}_{t_{i+1}}, t_{i+1})}_{\text{Propagated Error}}.
    \end{align}
    Taking the $\ell_1$-norm on both sides and applying the triangle inequality yields:
    \begin{align}
        \left\| \bm x_{t_i} - \hat{\bm x}_{t_i} \right\|_1
        & \le \left\| \bm x_{t_{i+1}} - \hat{\bm x}_{t_{i+1}} \right\|_1 \notag \\
        & \quad + h_{t_{i+1}} \left\| \bm v(\bm x_{t_{i+1}}, t_{i+1}) - \bm v_\theta(\bm x_{t_{i+1}}, t_{i+1}) \right\|_1 \notag \\
        & \quad + h_{t_{i+1}} \left\| \bm v_\theta( \bm x_{t_{i+1}}, t_{i+1}) - \bm v_\theta (\hat{\bm x}_{t_{i+1}}, t_{i+1}) \right\|_1 \notag \\
        & \quad + \frac{h_{t_{i+1}}^2}{2} \left\| \left. \frac{\mathrm{d}}{\mathrm{d}t} \bm v (\bm x_t, t) \right|_{\xi_{i+1}} \right\|_1.
    \end{align}
    Applying the Lipschitz continuity of $\bm{v}$ (with constant $L_{t_{i+1}}$), the approximation bound $\eta$, and the total derivative bound $M$ from Assumption~\ref{assump:regularity}, we have:
    \begin{align}
        \left\| \bm x_{t_i} - \hat{\bm x}_{t_i} \right\|_1
        & \le \left\| \bm x_{t_{i+1}} - \hat{\bm x}_{t_{i+1}} \right\|_1
        + h_{t_{i+1}} L_{t_{i+1}} \left\| \bm x_{t_{i+1}} - \hat{\bm x}_{t_{i+1}} \right\|_1
        + h_{t_{i+1}} \eta
        + \frac{h_{t_{i+1}}^2}{2} M \notag \\
        & = \left(1 + h_{t_{i+1}} L_{t_{i+1}}\right) \left\| \bm x_{t_{i+1}} - \hat{\bm x}_{t_{i+1}} \right\|_1
        + h_{t_{i+1}} \eta
        + \frac{h_{t_{i+1}}^2}{2} M,
    \end{align}
    which completes the proof.
\end{proof}

\begin{lemma}[Single-Step Propagation Error under Cache Reuse]
    \label{lemma: single-step update error cache reuse}
    Under Assumption~\ref{assump:regularity}, let $\bm{x}_{t}$ denote the state at time $t$ along the ground-truth ODE trajectory, and let $\hat{\bm{x}}_{t}$ be the state generated by the Euler solver.
    Suppose that at time $t_{i+1}$, we reuse a cached residual $\bm{\delta}_{t_{i+1}}^c$ such that the reuse error is $\bm{\epsilon}_{t_{i+1}}^c = \bm{\delta}_{t_{i+1}} - \bm{\delta}^c_{t_{i+1}}$.
    Let $\hat{\bm{x}}_{t_i}^c$ be the state at $t_i$ produced by an Euler step from $\hat{\bm{x}}_{t_{i+1}}$ using the cached velocity $\bm{v}_\theta^c$.
    For a backward step size $h_{t_{i+1}} = t_{i+1} - t_i > 0$, the error at time $t_i$ is bounded by:
    \begin{equation}
        \left\| \bm x_{t_i} - \hat{\bm x}^c_{t_i} \right\|_1
        \le \left(1 + h_{t_{i+1}} L_{t_{i+1}}\right) \left\| \bm x_{t_{i+1}} - \hat{\bm x}_{t_{i+1}} \right\|_1
        + h_{t_{i+1}} \eta
        + h_{t_{i+1}} L_\text{out} \left\| \bm \epsilon^c_{t_{i+1}} \right\|_1
        + \frac{h_{t_{i+1}}^2}{2} M.
    \end{equation}
\end{lemma}

\begin{proof}
    By Taylor's theorem, the ground-truth state $\bm x_{t_i}$ can be expanded about $t_{i+1}$ by:
    \begin{align}
        \bm x_{t_i}
        = \bm x_{t_{i+1}}
        - h_{t_{i+1}} \bm v(\bm x_{t_{i+1}}, t_{i+1})
        + \frac{h_{t_{i+1}}^2}{2} \left.\frac{\mathrm{d}}{\mathrm{d}t} \bm v(\bm x_t, t) \right|_{\xi_{i+1}},
    \end{align}
    where $\xi_{i+1} \in [t_i, t_{i+1}]$ is an intermediate time point arising from the Lagrange form of the remainder, and we use the relation $\frac{\mathrm{d}\bm x_t}{\mathrm{d}t} = \bm v(\bm x_t, t)$.

    Let $\bm v^c_\theta(\bm x_t, t)$ denote the learned velocity computed using the cached residual.
    The Euler update with cache reuse for $\hat{\bm x}_{t_i}$ is defined as:
    \begin{equation}
        \hat{\bm x}^c_{t_i}
        = \hat{\bm x}_{t_{i+1}}
        - h_{t_{i+1}}  \bm v^c_\theta(\hat{\bm x}_{t_{i+1}}, t_{i+1}).
    \end{equation}
    Subtracting the discrete update from the Taylor expansion, we have:
    \begin{align}
        \bm x_{t_i} - \hat{\bm x}^c_{t_i}
        = \bm x_{t_{i+1}} - \hat{\bm x}_{t_{i+1}}
        - h_{t_{i+1}} \left( \bm v(\bm x_{t_{i+1}}, t_{i+1}) - \bm v^c_\theta(\hat{\bm x}_{t_{i+1}}, t_{i+1}) \right)
        + \frac{h_{t_{i+1}}^2}{2} \left. \frac{\mathrm{d}}{\mathrm{d}t} \bm v(\bm x_t, t) \right|_{\xi_{i+1}}.
    \end{align}
    We decompose the velocity difference into three components to isolate the cache reuse error:
    \begin{align}
        \bm v(\bm x_{t_{i+1}}, t_{i+1}) - \bm v^c_\theta(\hat{\bm x}_{t_{i+1}}, t_{i+1})
        & = \underbrace{\bm v(\bm x_{t_{i+1}}, t_{i+1}) - \bm v_\theta(\bm x_{t_{i+1}}, t_{i+1})}_{\text{Approximation Error}} \notag \\
        & \quad + \underbrace{\bm v_\theta(\bm x_{t_{i+1}}, t_{i+1}) - \bm v_\theta(\hat{\bm x}_{t_{i+1}}, t_{i+1})}_{\text{Propagated Error}} \notag \\
        & \quad + \underbrace{\bm v_\theta(\hat{\bm x}_{t_{i+1}}, t_{i+1}) - \bm v^c_\theta(\hat{\bm x}_{t_{i+1}}, t_{i+1})}_{\text{Cache Reuse Error}}.
    \end{align}
    Taking the $\ell_1$-norm on both sides and applying the triangle inequality yields:
    \begin{align}
        \left\| \bm x_{t_i} - \hat{\bm x}^c_{t_i} \right\|_1
        & \le \left\| \bm x_{t_{i+1}} - \hat{\bm x}_{t_{i+1}} \right\|_1 \notag \\
        & \quad + h_{t_{i+1}} \left\| \bm v(\bm x_{t_{i+1}}, t_{i+1}) - \bm v_\theta(\bm x_{t_{i+1}}, t_{i+1}) \right\|_1 \notag \\
        & \quad + h_{t_{i+1}} \left\| \bm v_\theta(\bm x_{t_{i+1}}, t_{i+1}) - \bm v_\theta(\hat{\bm x}_{t_{i+1}}, t_{i+1}) \right\|_1 \notag \\
        & \quad + h_{t_{i+1}} \left\| \bm v_\theta(\hat{\bm x}_{t_{i+1}}, t_{i+1}) - \bm v^c_\theta(\hat{\bm x}_{t_{i+1}}, t_{i+1}) \right\|_1 \notag \\
        & \quad + \frac{h_{t_{i+1}}^2}{2} \left\| \left. \frac{\mathrm{d}}{\mathrm{d}t} \bm v (\bm x_t, t) \right|_{\xi_{i+1}} \right\|_1.
    \end{align}
    Given the model architecture $\bm{v}_\theta(\bm{x}_t, t) = f_\text{out} (\bm{h}_t + \bm{\delta}_t)$, where $\bm{h}_t$ is the feature and $\bm{\delta}_t$ is the residual, the cache reuse error term is bounded by the Lipschitz constant $L_\text{out}$:
    % Recalling the learned velocity definition $\bm v_\theta(\bm x_t, t) = f_\text{out} (\bm h_t + \bm \delta_t)$ (where $\bm h_t$ is the input feature and $\bm \delta_t$ is the residual), and noting that the cached velocity uses $\bm \delta^c_{t_{i+1}}$, the cache reuse error term becomes:
    \begin{equation}
        \left\| f_\text{out} (\bm h_{t_{i+1}} + \bm \delta_{t_{i+1}}) - f_\text{out} (\bm h_{t_{i+1}} + \bm \delta^c_{t_{i+1}}) \right\|_1
        \le L_\text{out} \left\| \bm \delta_{t_{i+1}} - \bm \delta^c_{t_{i+1}} \right\|_1
        = L_\text{out} \left\| \bm \epsilon^c_{t_{i+1}} \right\|_1.
    \end{equation}
    Finally, substituting the Lipschitz constant $L_{t_{i+1}}$, the approximation bound $\eta$, and the total derivative bound $M$ from Assumption~\ref{assump:regularity}, we obtain:
    % Finally, applying the Lipschitz continuity of $\bm{v}$ (with constant $L_{t_{i+1}}$), the approximation bound $\eta$, and the total derivative bound $M$ from Assumption~\ref{assump:regularity}, we obtain:
    \begin{align}
        \left\| \bm x_{t_i} - \hat{\bm x}_{t_i} \right\|_1
        & \le \left\| \bm x_{t_{i+1}} - \hat{\bm x}_{t_{i+1}} \right\|_1
        + h_{t_{i+1}} L_{t_{i+1}} \left\| \bm x_{t_{i+1}} - \hat{\bm x}_{t_{i+1}} \right\|_1
        + h_{t_{i+1}} \eta \notag \\
        & \quad \ + h_{t_{i+1}} L_\text{out} \left\| \bm \epsilon^c_{t_{i+1}} \right\|_1
        + \frac{h_{t_{i+1}}^2}{2} M \notag \\
        & = \left(1 + h_{t_{i+1}} L_{t_{i+1}}\right) \left\| \bm x_{t_{i+1}} - \hat{\bm x}_{t_{i+1}} \right\|_1
        + h_{t_{i+1}} \eta \notag \\
        & \quad \ + h_{t_{i+1}} L_\text{out} \left\| \bm \epsilon^c_{t_{i+1}} \right\|_1
        + \frac{h_{t_{i+1}}^2}{2} M.
    \end{align}
    This completes the proof.
\end{proof}

\begin{manualtheorem}{\ref{theorem:error-bound_single-step_PF-ODE}}[Global Error Bound under Cache Reuse]
    Under Assumption~\ref{assump:regularity}, let $\bm{x}_{t}$ be the ground-truth ODE state at time $t$, and let $\hat{\bm{x}}^c_t$ be the state generated by the Euler solver using the learned velocity $\bm{v}_\theta$, incorporating a single cache reuse event.
    Specifically, suppose that at step $t_{i+1} \to t_i$, we reuse a cached residual $\bm{\delta}_{t_{i+1}}^c$ with error $\bm{\epsilon}_{t_{i+1}}^c = \bm{\delta}_{t_{i+1}} - \bm{\delta}^c_{t_{i+1}}$.
    Given a sequence of timesteps $t_{N}, t_{N-1}, \dots, t_1$ with uniform step size $h$, the global error at the final timestep $t_1$ is bounded by:
    \begin{align}
        \left\| \bm{x}_{t_1} - \hat{\bm x}^c_{t_1} \right\|_1
        & \le \underbrace{ h L_\text{out} \left\| \bm \epsilon^c_{t_{i+1}} \right\|_1 e^{(i-1)hL} }_{\text{Cache Reuse Error}} 
        + \underbrace{ \left( \frac{\eta}{L} + \frac{h M}{2 L} \right) \left( e^{(N-1) h L} - 1\right) }_{ \text{Approximation \& Discretization Error} }.
    \end{align}
\end{manualtheorem}

\begin{proof}
    We analyze the error accumulation by recursively applying the single-step error bounds.
    From Lemma~\ref{lemma: single-step update error}, the error at the final timestep $t_1$ can be expressed in terms of the error at an intermediate step $t_i$:
    \begin{equation}
        \left\| \bm{x}_{t_1} - \hat{\bm x}^c_{t_1} \right\|_1
        \le \left\| \bm x_{t_{i}} - \hat{\bm x}^c_{t_{i}} \right\|_1 (1 + h L)^{i-1}
        % + h \, \eta \sum_{n=1}^{i-1} (1 + h L)^{n-1}
        + \left( h \, \eta + \frac{h^2}{2} M \right) \sum_{n=1}^{i-1}  (1 + h L)^{n-1}.
    \end{equation}
    % Applying Lemma~\ref{lemma: single-step update error} on $\| \bm x_{t_i} - \hat{\bm x}_{t_i} \|$ recursively from $t_{i}$ to $t_1$ yields
    % \begin{equation}
    %     \left\| \bm{x}_{t_1} - \hat{\bm x}_{t_1} \right\|_1
    %     \le \left\| \bm x_{t_{i}} - \hat{\bm x}_{t_{i}} \right\|_1 \prod_{j=1}^{i-1} \left( 1 + h_{t_{j+1}} L_{t_{j+1}} \right) + \sum_{j=1}^{i-1} \frac{h^2_{j+1}}{2} M \prod_{k=j}^{i-2} \left( 1 + h_{t_{k+1}} L_{t_{k+1}} \right).
    % \end{equation}
    At the specific transition $t_{i+1} \to t_i$ where cache reuse occurs, Lemma~\ref{lemma: single-step update error cache reuse} provides the local bound:
    \begin{equation}
        \left\| \bm x_{t_i} - \hat{\bm x}^c_{t_i} \right\|_1
        \le (1 + h L) \left\| \bm x_{t_{i+1}} - \hat{\bm x}_{t_{i+1}} \right\|_1
        + h \, \eta
        + h L_\text{out} \left\| \bm \epsilon^c_{t_{i+1}} \right\|_1
        + \frac{h^2}{2} M.
    \end{equation}
    % Applying Lemma~\ref{lemma: single-step update error} on $\| \bm x_{t_N} - \hat{\bm x}_{t_N} \|$ recursively from $t_{N}$ to $t_{i+1}$ yields
    % \begin{equation}
    %     \left\| \bm{x}_{t_{i+1}} - \hat{\bm x}_{t_{i+1}} \right\|_1
    %     \le \left\| \bm x_{t_{N+1}} - \hat{\bm x}_{t_{N+1}} \right\|_1 \prod_{j=i+1}^{N} \left( 1 + h_{t_{j+1}} L_{t_{j+1}} \right) + \sum_{j=i+1}^{N} \frac{h^2_{j+1}}{2} M \prod_{k=j}^{N-1} \left( 1 + h_{t_{k+1}} L_{t_{k+1}} \right).
    % \end{equation}
    % From Lemma~\ref{lemma: single-step update error cache reuse}, we have
    % \begin{equation}
    %     \left\| \bm x_{t_i} - \hat{\bm x}_{t_i} \right\|_1
    %     \le \left(1 + h_{t_{i+1}} L_{t_{i+1}}\right) \left\| \bm x_{t_{i+1}} - \hat{\bm x}_{t_{i+1}} \right\|_1
    %     + h_{t_{i+1}} t_{i+1} L_\text{out} \left\| \bm \epsilon^c_{t_{i+1}} \right\|_1
    %     + \frac{h_{t_{i+1}}^2}{2} M.
    % \end{equation}
    Continuing the recursion for the remaining steps from $t_{i+1}$ back to the initial state $t_{N}$, we observe that for all other steps $n \in \{i+1, \dots, N-1\}$, the standard local error bound from Lemma~\ref{lemma: single-step update error} applies.
    Combining these, the global error at $t_1$ becomes:
    \begin{align}
        \left\| \bm{x}_{t_1} - \hat{\bm x}^c_{t_1} \right\|_1
        \le \left\| \bm x_{t_{N}} - \hat{\bm x}_{t_{N}} \right\|_1 \left( 1 + h L \right)^{N-1}
        + h \, L_\text{out} \left\| \bm \epsilon^c_{t_{i+1}} \right\|_1 \left( 1 + h L \right)^{i-1} \notag \\
        + \left( h \, \eta + \frac{h^2}{2} M \right) \sum_{n=1}^{N-1} (1 + h L)^{n-1}.
    \end{align}
    Assuming the solver starts from the ground-truth initial noise, we have $\bm{x}_{t_N} = \hat{\bm{x}}_{t_N}$ and the first term vanishes.
    The remaining terms simplify to:
    \begin{align}
        \left\| \bm{x}_{t_1} - \hat{\bm x}^c_{t_1} \right\|_1
        \le h \, L_\text{out} \left\| \bm \epsilon^c_{t_{i+1}} \right\|_1 \left( 1 + h L \right)^{i-1}
        + \left( h \, \eta + \frac{h^2}{2} M \right) \sum_{n=1}^{N-1} (1 + h L)^{n-1}.
    \end{align}
    Using the geometric series identity $\sum_{j=0}^{k-1} r^j = \frac{r^k - 1}{r - 1}$ with $r = 1 + hL$, the remaining terms simplifies to:
    \begin{align}
        \left\| \bm{x}_{t_1} - \hat{\bm x}^c_{t_1} \right\|_1
        & \le h L_\text{out} \left\| \bm \epsilon^c_{t_{i+1}} \right\|_1 (1 + h L)^{i-1}
        + \left( h \eta + \frac{h^2}{2} M \right) \frac{1}{h L} \left( (1 + h L)^{N-1} - 1 \right) \notag \\
        & = h L_\text{out} \left\| \bm \epsilon^c_{t_{i+1}} \right\|_1 (1 + h L)^{i-1}
        + \left( \frac{\eta}{L} + \frac{h M}{2L} \right) \left( (1 + h L)^{N-1} - 1 \right)
    \end{align}
    Finally, employing the inequality $1 + a \le e^{a}$, we arrive at the final bound:
    \begin{align}
        \left\| \bm{x}_{t_1} - \hat{\bm x}^c_{t_1} \right\|_1
        & \le h L_\text{out} \left\| \bm \epsilon^c_{t_{i+1}} \right\|_1 e^{(i-1)hL}
        + \left( \frac{\eta}{L} + \frac{h M}{2 L} \right) \left( e^{(N-1) h L} - 1\right).
    \end{align}
    This completes the proof.
\end{proof}

\subsection{Proof for \cref{theorem:error-bound_single-step_Euler}}

\begin{manualtheorem}{\ref{theorem:error-bound_single-step_Euler}}[Error Propagation Bound under Single-Step Cache Reuse]
    Under Assumption~\ref{assump:regularity}, let $\hat{\bm{x}}_t$ be the state generated by the Euler solver using the learned velocity $\bm{v}_\theta$, and let $\hat{\bm{x}}_t^c$ be the state incorporating a single cache reuse event at step $t_{i+1} \to t_i$. 
    Suppose the reuse of a cached residual $\bm{\delta}_{t_{i+1}}^c$ introduces an error $\bm{\epsilon}_{t_{i+1}}^c = \bm{\delta}_{t_{i+1}} - \bm{\delta}^c_{t_{i+1}}$.
    Given a sequence of timesteps $t_{N}, t_{N-1}, \dots, t_1$ with step sizes $h_{t_{n+1}} = t_{n+1} - t_n$, the cumulative error at the final timestep $t_1$ is bounded by:
    \begin{equation}
        \left\| \hat{\bm x}_{t_1} - \hat{\bm x}^c_{t_1} \right\|_1
        \le \left\| \bm \epsilon^c_{t_{i+1}} \right\|_1 e^{w_{t_{i+1}}},
    \end{equation}
    where the propagation exponent $w_{t_{i+1}}$ is defined as:
    \begin{equation}
        w_{t_{i+1}}
        = \ln \left( h_{t_{i+1}} L_{\text{out}} \right)
        + \sum_{n=1}^{i-1} h_{t_{n+1}} L_{t_{n+1}}.
    \end{equation}
\end{manualtheorem}

\begin{proof}
    We first bound the error propagation for the steps following the cache reuse ($t_n$ for $n < i$).
    The Euler updates for the standard and cached trajectories are:
    \begin{align}
        \hat{\bm{x}}_{t_n}  
        & = \hat{\bm{x}}_{t_{n+1}} - h_{t_{n+1}} \bm{v}_\theta(\hat{\bm{x}}_{t_{n+1}}, t_{n+1}), \\
        \hat{\bm{x}}^c_{t_n}
        & = \hat{\bm{x}}^c_{t_{n+1}} - h_{t_{n+1}} \bm{v}_\theta(\hat{\bm{x}}^c_{t_{n+1}}, t_{n+1}).
    \end{align}
    Subtracting these updates and applying the triangle inequality:
    \begin{align}
        \left\| \hat{\bm{x}}_{t_n} - \hat{\bm{x}}^c_{t_n} \right\|_1
        & = \left\| \hat{\bm{x}}_{t_{n+1}} - \hat{\bm{x}}_{t_{n+1}}^c - h_{t_{n+1}} \left( \bm{v}_\theta ( \hat{\bm{x}}_{t_{n+1}}, t_{n+1} ) - \bm{v}_\theta ( \hat{\bm{x}}^c_{t_{n+1}}, t_{n+1} ) \right) \right\|_1 \notag \\
        & \le \left\| \hat{\bm{x}}_{t_{n+1}} - \hat{\bm{x}}_{t_{n+1}}^c \right\|_1 + h_{t_{n+1}} \left\| \bm{v}_\theta ( \hat{\bm{x}}_{t_{n+1}}, t_{n+1} ) - \bm{v}_\theta ( \hat{\bm{x}}^c_{t_{n+1}}, t_{n+1} ) \right\|_1 \notag \\
        & \le (1 + h_{t_{n+1}} L_{t_{n+1}}) \left\| \hat{\bm{x}}_{t_{n+1}} - \hat{\bm{x}}_{t_{n+1}}^c \right\|_1,
    \end{align}
    where $L_{t_{n+1}}$ is the Lipschitz constant of $\bm{v}_\theta$.
    Applying this relation recursively from $t_1$ back to $t_i$ yields:
    \begin{equation}
        \left\| \hat{\bm{x}}_{t_1} - \hat{\bm{x}}^c_{t_1} \right\|_1 \le \left\| \hat{\bm{x}}_{t_i} - \hat{\bm{x}}^c_{t_i} \right\|_1 \prod_{n=1}^{i-1} (1 + h_{t_{n+1}} L_{t_{n+1}}).
    \end{equation}
    Next, we bound the local error injected at step $t_i$.
    At this step, both trajectories originate from the same state $\hat{\bm{x}}_{t_{i+1}}$, but the cached trajectory uses the approximate velocity $\bm{v}^c_\theta$.
    The updates are:
    \begin{align}
        \hat{\bm{x}}_{t_i}   &= \hat{\bm{x}}_{t_{i+1}} - h_{t_{i+1}} \bm{v}_\theta(\hat{\bm{x}}_{t_{i+1}}, t_{i+1}), \\
        \hat{\bm{x}}^c_{t_i} &= \hat{\bm{x}}_{t_{i+1}} - h_{t_{i+1}} \bm{v}^c_\theta(\hat{\bm{x}}_{t_{i+1}}, t_{i+1}).
    \end{align}
    The norm of the difference is thus:
    \begin{equation}
        \left\| \hat{\bm{x}}_{t_i} - \hat{\bm{x}}^c_{t_i} \right\|_1 = h_{t_{i+1}} \left\| \bm{v}_\theta(\hat{\bm{x}}_{t_{i+1}}, t_{i+1}) - \bm{v}^c_\theta(\hat{\bm{x}}_{t_{i+1}}, t_{i+1}) \right\|_1.
    \end{equation}
    Using the model architecture $\bm{v}_\theta(\bm{x}, t) = f_\text{out}(\bm{h} + \bm{\delta})$ and the Lipschitz continuity of the output layer $f_\text{out}$, we have:
    \begin{equation}
        \left\| f_\text{out} (\bm{h}_{t_{i+1}} + \bm{\delta}_{t_{i+1}}) - f_\text{out} (\bm{h}_{t_{i+1}} + \bm{\delta}^c_{t_{i+1}}) \right\|_1 \le L_\text{out} \left\| \bm{\epsilon}^c_{t_{i+1}} \right\|_1.
    \end{equation}
    Substituting this into the local error bound yields:
    \begin{equation}
        \left\| \hat{\bm{x}}_{t_i} - \hat{\bm{x}}^c_{t_i} \right\|_1 \le h_{t_{i+1}} L_\text{out} \left\| \bm{\epsilon}^c_{t_{i+1}} \right\|_1.
    \end{equation}
    Combining the propagation and local error bounds, and applying $1 + a \le e^a$:
    \begin{align}
        \left\| \hat{\bm{x}}_{t_1} - \hat{\bm{x}}^c_{t_1} \right\|_1 
        &\le h_{t_{i+1}} L_\text{out} \left\| \bm{\epsilon}^c_{t_{i+1}} \right\|_1 \prod_{n=1}^{i-1} \exp(h_{t_{n+1}} L_{t_{n+1}}) \notag \\
        &= \left\| \bm{\epsilon}^c_{t_{i+1}} \right\|_1 \exp \left( \ln(h_{t_{i+1}} L_\text{out}) + \sum_{n=1}^{i-1} h_{t_{n+1}} L_{t_{n+1}} \right).
    \end{align}
    This completes the proof.
\end{proof}

\subsection{Proof for \cref{theorem:error-bound_multi-step_Euler}}

\begin{manualtheorem}{\ref{theorem:error-bound_multi-step_Euler}}[Error Propagation Bound under Multi-Step Cache Reuse]
    Under Assumption~\ref{assump:regularity}, let $\hat{\bm{x}}_t$ be the state generated by the Euler solver using the learned velocity $\bm{v}_\theta$, and let $\hat{\bm{x}}_t^c$ be the state incorporating multiple cache reuse events at every step $t_{n+1} \to t_n$ for $n \in \{N-1, \dots, 1\}$. 
    Suppose each reuse of a cached residual introduces a local error $\bm{\epsilon}_{t_{n+1}}^c = \bm{\delta}_{t_{n+1}} - \bm{\delta}^c_{t_{n+1}}$.
    Given a sequence of timesteps $t_{N}, t_{N-1}, \dots, t_1$ with step sizes $h_{t_{n+1}} = t_{n+1} - t_n$, the cumulative error at the final timestep $t_1$ is bounded by:
    \begin{equation}
        \left\| \hat{\bm{x}}_{t_1} - \hat{\bm{x}}^c_{t_1} \right\|_1
        \le \sum_{n=1}^{N-1} \left\| \bm \epsilon^c_{t_{n+1}} \right\|_1 e^{w_{t_{n+1}}}.
    \end{equation}
    where the propagation exponent $w_{t_{n+1}}$ is defined as:
    \begin{equation}
        w_{t_{n+1}}
        = \ln \left( h_{t_{n+1}} L_\text{out} \right)
        + \sum_{m=1}^{n-1} h_{t_{m+1}} L_{t_{m+1}}.
    \end{equation}
\end{manualtheorem}

\begin{proof}
    % We first bound the local error injected at step $t_{n+1} \to t_n$ for $n \le N-1$.
    % At these steps, both trajectories originate from the same state $\hat{\bm{x}}^c_{t_{n+1}}$, but the cached trajectory uses the approximate velocity $\bm{v}^c_\theta$.
    % The updates are:
    We begin by analyzing the error transition for a single step $t_{n+1} \to t_n$.
    The standard Euler update and the cached Euler update are given by:
    \begin{align}
        \hat{\bm{x}}_{t_n}
        & = \hat{\bm{x}}_{t_{n+1}} - h_{t_{n+1}} \bm{v}_\theta(\hat{\bm{x}}_{t_{n+1}}, t_{n+1}), \\
        \hat{\bm{x}}^c_{t_n}
        & = \hat{\bm{x}}^c_{t_{n+1}} - h_{t_{n+1}} \bm{v}^c_\theta(\hat{\bm{x}}^c_{t_{n+1}}, t_{n+1}).
    \end{align}
    Subtracting these updates, we have:
    \begin{equation}
        \hat{\bm{x}}_{t_n} - \hat{\bm{x}}^c_{t_n}
        = \hat{\bm{x}}_{t_{n+1}} - \hat{\bm{x}}^c_{t_{n+1}} - h_{t_{n+1}} \left( \bm{v}_\theta(\hat{\bm{x}}_{t_{n+1}}, t_{n+1}) - \bm{v}^c_\theta(\hat{\bm{x}}^c_{t_{n+1}}, t_{n+1}) \right).
    \end{equation}
    We decompose the velocity difference into two components to isolate the cache reuse error:
    \begin{align}
        \bm{v}_\theta(\hat{\bm{x}}_{t_{n+1}}, t_{n+1}) - \bm{v}^c_\theta(\hat{\bm{x}}^c_{t_{n+1}}, t_{n+1})
        & = \underbrace{ \bm{v}_\theta(\hat{\bm{x}}_{t_{n+1}}, t_{n+1}) - \bm{v}_\theta( \hat{\bm{x}}^c_{t_{n+1}}, t_{n+1} )}_{\text{Propagated Error}} \notag \\
        & \quad \ + \underbrace{ \bm{v}_\theta( \hat{\bm{x}}^c_{t_{n+1}}, t_{n+1} ) - \bm{v}^c_\theta(\hat{\bm{x}}^c_{t_{n+1}}, t_{n+1})}_{\text{Cache Reuse Error}}.
    \end{align}
    Applying the Lipschitz constant $L_{t_{n+1}}$ for the velocity $\bm{v}_\theta$ and $L_{\text{out}}$ for the output layer $f_{\text{out}}$ with the triangle inequality, we obtain the recurrence relation:
    \begin{align}
        \left\| \hat{\bm{x}}_{t_n} - \hat{\bm{x}}^c_{t_n} \right\|_1
        & \le \left\| \hat{\bm{x}}_{t_{n+1}} - \hat{\bm{x}}^c_{t_{n+1}} \right\|_1 \notag \\
        & \quad + h_{t_{n+1}} \left\| \bm{v}_\theta( \hat{\bm{x}}_{t_{n+1}}, t_{n+1} ) - \bm{v}_\theta( \hat{\bm{x}}^c_{t_{n+1}}, t_{n+1} ) \right\|_1 \notag \\
        & \quad + h_{t_{n+1}} \left\| \bm{v}_\theta( \hat{\bm{x}}^c_{t_{n+1}}, t_{n+1} ) - \bm{v}^c_\theta(\hat{\bm{x}}^c_{t_{n+1}}, t_{n+1}) \right\|_1 \\
        & \le \left( 1 + h_{t_{n+1}} L_{t_{n+1}} \right) \left\| \hat{\bm{x}}_{t_{n+1}} - \hat{\bm{x}}^c_{t_{n+1}} \right\|_1
        + h_{t_{n+1}} L_\text{out} \left\| \bm \epsilon^c_{t_{n+1}} \right\|_1.
    \end{align}
    Applying this relation recursively from $t_1$ back to $t_N$ yields:
    \begin{align}
        \left\| \hat{\bm{x}}_{t_1} - \hat{\bm{x}}^c_{t_1} \right\|_1
        & \le \left\| \hat{\bm{x}}_{t_{N}} - \hat{\bm{x}}^c_{t_{N}} \right\|_1 \prod_{n=1}^{N-1} \left( 1 + h_{t_{n+1}} L_{t_{n+1}} \right) \notag \\
        & \quad \ + \sum_{n=1}^{N-1} h_{t_{n+1}} L_\text{out} \left\| \bm \epsilon^c_{t_{n+1}} \right\|_1 \prod_{m=1}^{n-1} \left( 1 + h_{t_{m+1}} L_{t_{m+1}} \right).
    \end{align}
    Since both trajectories start from the same initial noise, we have $\hat{\bm{x}}_{t_N} = \hat{\bm{x}}^c_{t_N}$ and the first term vanishes,
    \begin{equation}
        \left\| \hat{\bm{x}}_{t_1} - \hat{\bm{x}}^c_{t_1} \right\|_1
        \le \sum_{n=1}^{N-1} h_{t_{n+1}} L_\text{out} \left\| \bm \epsilon^c_{t_{n+1}} \right\|_1 \prod_{m=1}^{n-1} \left( 1 + h_{t_{m+1}} L_{t_{m+1}} \right).
    \end{equation}
    Finally, using the inequality $1 + a \le e^a$,
    \begin{equation}
        \left\| \hat{\bm{x}}_{t_1} - \hat{\bm{x}}^c_{t_1} \right\|_1
        \le \sum_{n=1}^{N-1} \left\| \bm \epsilon^c_{t_{n+1}} \right\|_1 \exp \left( \ln \left( h_{t_{n+1}} L_\text{out} \right) + \sum_{m=1}^{n-1} h_{t_{m+1}} L_{t_{m+1}} \right).
    \end{equation}
    This completes the proof.
\end{proof}

% \section{Optimization on Lower Frame Rates and Resolutions}

% \newpage

\section{Related Work}
\label{app:rel_work}

\subsection{Diffusion Models}

In recent years, diffusion models~\cite{DBLP:conf/nips20:DDPM, DBLP:conf/iclr21:ScoreSDE, DBLP:conf/nips21:DiffusionBeatGAN, DBLP:conf/iclr23:FlowMatching} have emerged as a dominant paradigm in generative modeling, demonstrating remarkable performance across a wide range of visual generation tasks, including image~\cite{DBLP:conf/cvpr22:StableDiffussion, DBLP:conf/nips22:imagen}, video~\cite{DBLP:journals/corr23/StableVideoDiffusion, DBLP:conf/cvpr24:VideoCrafter2}, and multimodal synthesis~\cite{nips25:llada}.
Early diffusion models were primarily built upon lightweight U-Net architectures~\cite{DBLP:conf/miccai15:Unet}, which enabled stable training and strong generative fidelity through iterative denoising.
Subsequent advances have significantly expanded model capacity by introducing transformer-based backbones, most notably Diffusion Transformers (DiTs)~\cite{DBLP:conf/iccv23:DiT}, which facilitate better scalability and have driven substantial progress in high-resolution and long-horizon generation, especially for video synthesis.
Despite their strong generative capabilities, diffusion models remain computationally expensive at inference time.
The sequential denoising process, combined with increasingly large backbones and longer generation horizons, results in slow sampling speed and high computational cost~\cite{DBLP:journals/corr24:Delta-DiT}.
These limitations pose significant challenges for practical deployment in real-world scenarios, motivating extensive research on accelerating diffusion model inference.

\subsection{Diffusion Model Acceleration}

A large body of work has investigated accelerating diffusion model inference from different perspectives. 
A representative line of research focuses on improving the sampling solvers, aiming to reduce the number of denoising steps without modifying the underlying model.
Classic approaches, such as DDIM~\cite{DBLP:conf/iclr21:ddim}, DPM-Solver~\cite{DBLP:conf/nips22:DPM-Solver}, UniPC~\cite{DBLP:conf/nips23:UniPC}, etc., have demonstrated that carefully designed numerical solvers can significantly accelerate sampling.
As a result, modern high-fidelity image and video generation models (e.g., Wan~\cite{DBLP:journals/corr25:wan2.1}, OpenSora~\cite{DBLP:journals/corr24:opensora}, Flux-dev 1.0~\cite{flux2024}) already adopt these advanced general-purpose solvers in practice.
Nevertheless, despite such solver-level optimizations, tens of denoising steps are still required for high-quality generation, leading to substantial inference cost.
Another direction explores distillation-based acceleration, where a compact student model is trained to mimic a large teacher diffusion model using fewer steps.
While such methods can achieve impressive speedups, they typically require expensive retraining and large-scale supervision.
This limitation becomes particularly pronounced for video diffusion models, where both training and distillation incur prohibitive computational overhead~\cite{DBLP:conf/cvpr23:DiffusionMadeSlim}.
Model-level optimization constitutes another important category, including quantization and pruning.
Quantization methods such as PTQD~\cite{DBLP:conf/nips23:PTQD} reduce numerical precision to lower computation and memory cost, while pruning-based approaches like LD-Pruner~\cite{DBLP:conf/cvpr24:LD-Pruner} remove redundant structures from diffusion models.
Although effective, these techniques often involve nontrivial engineering effort, additional calibration, or task-specific retraining, and their performance can be sensitive to model architecture and deployment settings.
More recently, \emph{cache-based acceleration} has emerged as a promising and complementary direction~\cite{DBLP:journals/corr24:Delta-DiT, DBLP:conf/cvpr25:TeaCache, DBLP:journals/corr25:ERTACache}.
These methods exploit temporal redundancy in the reverse diffusion process by reusing intermediate representations or residuals across adjacent timesteps, without modifying the solver or retraining the backbone model.
As such, cache-based techniques are largely orthogonal to solver design and can be seamlessly applied on top of existing solver-based samplers (e.g., DDIM/UniPC), enabling further acceleration in a plug-and-play manner.
By reusing computation from previous denoising steps, cache-based approaches can substantially reduce inference cost while preserving generation quality, which is particularly attractive for large-scale and high-fidelity image and video diffusion models.

\section{Experimental Details}
\label{app:exp_details}

\subsection{Evaluation Prompts}
\label{app:evaluation-prompts}

For video generation, we follow prior work~\cite{DBLP:conf/cvpr25:TeaCache,DBLP:journals/corr25:ERTACache} and evaluate all methods using the official 946 prompts provided by VBench~\cite{DBLP:conf/cvpr24:vbench}. These prompts cover a diverse set of content categories and motion patterns, and are designed to comprehensively assess video generation quality across multiple dimensions.

For image generation, we adopt the official COCO validation set~\cite{DBLP:conf/eccv14:coco} and use the first 30K text prompts as commonly done in recent diffusion acceleration studies~\cite{DBLP:conf/cvpr25:TeaCache,DBLP:journals/corr25:ERTACache}. The prompt list is publicly available and released on Hugging Face to facilitate reproducibility. All compared methods are evaluated on the same prompt sets to ensure a fair comparison.

\subsection{Evaluation Metrics}
\label{app:evaluation-metrics}

We employ three evaluation metrics to assess generation quality and fidelity.
\begin{itemize}
    [leftmargin=1.2em, topsep=-0.5pt, itemsep=3pt, parsep=3pt]
    \item \textbf{VBench}~\cite{DBLP:conf/cvpr24:vbench} serves as a holistic benchmarking framework for video generative models. It utilizes a hierarchical Evaluation Dimension Suite to disentangle the multifaceted nature of ``video quality'' into distinct, well-defined metrics, thereby enabling a granular and objective assessment of generative performance.
    \item \textbf{LPIPS}~\cite{DBLP:conf/cvpr18:LPIPS} quantifies perceptual similarity using deep feature representations, capturing subtle texture-level and semantic deviations.
    \item \textbf{PSNR} and \textbf{SSIM} measure pixel-level and structure-level fidelity, respectively, between outputs produced by the accelerated sampler and those generated by the corresponding base model.
\end{itemize}

\subsection{Training Details}
\label{app:training-details}

In this section, we provide the technical details for the bilevel optimization process of GCache.
The optimization is designed to be efficient while ensuring the generalizability of the learned propagation exponent.

\textbf{Dataset and Sampling Strategy}.
We utilize a training set consisting of 512 prompts randomly sampled from the corresponding dataset.
To ensure the robustness of the learned policy and avoid overfitting, we strictly ensure that the random seeds used during training are distinct from those used in the evaluation and testing phases.
During each iteration of the outer optimization, we evaluate the current caching policy using a batch size of 32 prompts to obtain a stable estimate of the perceptual loss (LPIPS and SSIM).

\textbf{Bayesian Optimization Configuration}.
We employ a Gaussian Process (GP) as the surrogate model for the outer objective. The optimization parameters and strategies are configured as follows:
\begin{itemize}
    [leftmargin=1.2em, topsep=-0.5pt, itemsep=3pt, parsep=3pt]
    \item \textbf{Search Space}: The learnable coefficients $\bm s$ for the Bernstein polynomial are bounded within the range $[\bm s_\text{min}, \bm s_\text{max}] = [0, 10]$.
    For our experiments, we set the polynomial degree $d = 3$, resulting in a 4-dimensional parameter space ($d+1 = 4$).
    \item \textbf{Structured Initialization}: Rather than random sampling, we initialize the GP surrogate using 16 deterministic points to ensure comprehensive coverage of the search space.
    Specifically, we define two representative centers $\{0.25, 0.75\}$ in the normalized parameter space for each dimension and generate the initial set via a Cartesian product across all 4 dimensions ($2^4 = 16$).
    These normalized coordinates are then linearly mapped to the actual range $[0, 10]$.
    This symmetric grid initialization ensures that the GP begins with a well-distributed understanding of the objective landscape across different quadrants.
    \item \textbf{Optimization Steps}: The total optimization budget is set to 500 steps.
    \item \textbf{Exploration-Exploitation Trade-off}: We utilize the Lower Confidence Bound (LCB) acquisition function: $A_{\text{LCB}}(\mathbf{s}) = \mu(\mathbf{s}) - \kappa \sigma(\mathbf{s})$.
    The exploration parameter $\kappa$ is scheduled to decay linearly from 2.576 to 1.0 over the course of training.
    This encourages the optimizer to prioritize global exploration in early iterations and transition towards local exploitation of identified high-quality regions in the later stages.
\end{itemize}

\textbf{Efficient Error Pre-calculation}.
To minimize the computational overhead during the bilevel optimization, we implement an efficient evaluation strategy for the inner objective.
Specifically, we pre-calculate the local approximation errors $\bm \epsilon^c_{i+1}$ across 128 representative samples before the start of the optimization.
Since the inner objective (the Dynamic Programming solver) only requires the magnitude of these local errors to calculate the total propagated error $\sum \| \epsilon \|_1 e^{w(t;\bm s)}$, pre-storing these values allows the optimization process to bypass redundant model evaluations.
During each step of the bilevel search, we simply scale the stored error values by the updated propagation exponent $e^{w(t;\bm s)}$.
This decoupling of error measurement from policy search significantly accelerates the optimization, reducing the search time from hours to minutes on a single GPU.
Furthermore, to verify the fidelity of these pre-computed proxies, we evaluate the alignment between errors from the original and cache-reused trajectories in Appendix~\ref{app:exp_results:precompute_error}.
The results justify the use of pre-computed errors as a reliable and high-fidelity proxy for policy optimization.

% \textbf{Low-resulotion Optimization Strategy}.
% To further reduce the computational overhead during the optimization, instead of optimization on 
% \YXC{TODO}

\subsection{Optimization Efficiency}
\label{app:optimization-efficiency}

\begin{table}[tbp]
    \centering
    \caption{
        We report the budget $K$ and the wall-clock time required.
        All policies were optimized in less than 24 hours, demonstrating the high efficiency of GCache.
    }
    \label{tab:optim_cost}
    \vskip 0.1in
    \footnotesize
    \begin{tabular}{cccc} % l=左对齐, c=居中
        \toprule
        \textbf{Model} & \textbf{GPU Config} & \textbf{Budget ($K$)} & \textbf{Time Cost} (hours) \\ 
        \midrule
        \multirow{2}{*}{OpenSora 1.2} & \multirow{2}{*}{4 $\times$ H100} & 18 & 17 \\
        & & 11 & 12.5 \\
        \midrule
        \multirow{2}{*}{CogVideoX-2B} & \multirow{2}{*}{4 $\times$ H100} & 31 & 7 \\
        & & 17 & 5 \\
        \midrule
        \multirow{2}{*}{Wan2.1-1.3B} & \multirow{2}{*}{8 $\times$ H100} & 24 & 17.5 \\
        & & 16 & 13.5 \\
        \midrule
        \multirow{2}{*}{Flux-dev 1.0} & \multirow{2}{*}{4 $\times$ H100} & 14 & 6 \\
        & & 10 & 5 \\
        \bottomrule
    \end{tabular}
\end{table}

\cref{tab:optim_cost} reports the optimization time and associated budgets for GCache across different architectures.
Experiments were performed using 4 $\times$ H100 GPUs for most models, with 8 $\times$ H100 GPUs reserved for Wan2.1-1.3B.
Crucially, GCache demonstrates remarkable efficiency: even for the most computationally demanding backbones, the optimization is finalized in less than a day.
This minimal overhead makes our approach highly scalable and suitable for the rapid deployment of new diffusion backbones.

\subsection{Experimental Details for Main Paper Visualizations}

In this section, we provide the specific experimental configurations used to generate the visualizations in the main manuscript:
\begin{itemize}
    \item \textbf{\cref{fig:intro}:} Results are evaluated using the Flux-dev 1.0~\cite{flux2024} backbone. Metrics are averaged over 32 randomly sampled prompts from the COCO~\cite{DBLP:conf/eccv14:coco} validation set.
    \item \textbf{\cref{fig:motivation}:} The error propagation analyses are derived from CogVideoX-2B~\cite{DBLP:conf/iclr25:CogVideoX}, with results averaged across 32 representative prompts from the VBench~\cite{DBLP:conf/cvpr24:vbench} suite.
    \item \textbf{\cref{fig:cache_policy_error_propagation}:} The policy comparison study is conducted on Flux-dev 1.0~\cite{flux2024}, utilizing 32 random prompts from the COCO~\cite{DBLP:conf/eccv14:coco} validation set for statistical consistency.
\end{itemize}

\section{Additional Experiment Results}
\label{app:exp_results}

\subsection{Robustness to Prompt Distribution Shifts}
\label{app:exp_results:prompt_distribution}

To evaluate the generalization capability of GCache under distribution shifts, we investigate whether a policy optimized on a specific prompt characteristic (e.g., static scenes) can generalize to others (e.g., highly dynamic videos).
Using VBench \cite{DBLP:conf/cvpr24:vbench} as a base, we curate two distinct subsets:
\begin{itemize}
[leftmargin=1.2em, topsep=-0.5pt, itemsep=3pt, parsep=3pt]
    \item \textbf{Static Prompts:} Scenes with minimal temporal evolution, identified by keywords such as ``in a still frame'', ``a tranquil tableau'', ``frozen in time'', ``static view.''
    \item \textbf{Dynamic Prompts:} Motion-intensive sequences sampled from VBench's motion-related dimensions, including \textit{human action}, \textit{dynamic degree}, \textit{motion smoothness}, and \textit{subject consistency}.
\end{itemize}
We optimize GCache on three training distributions: Static, Dynamic, and a Mixed (Random) set, and evaluate their cross-distribution performance.
All of the experiments are conducted on CogVideoX-2B with $K=17$.

\begin{table}[tbp]
    \centering
    \caption{
        Cross-distribution generalization analysis.
        We evaluate GCache policies optimized on different prompt subsets (Static vs. Dynamic).
        The negligible performance gap across training distributions demonstrates the robust generalization of our learned policy.
    }
    \vskip 0.1in
    \label{tab:comparison_prompt_distribution}
    \footnotesize
    \begin{tabular}{ccccc}
        \toprule
        \textbf{Test Set} & \textbf{Train Source} & \textbf{LPIPS$\downarrow$} & \textbf{SSIM$\uparrow$} & \textbf{PSNR$\uparrow$} \\ 
        \midrule
        \multirow{3}{*}{Static} & Mixed  & 0.0594 & 0.9122 & 30.25 \\
        & Static  & 0.0595 & 0.9118 & 30.26 \\
        & Dynamic & 0.0611 & 0.9109 & 30.26 \\ 
        \midrule
        \multirow{3}{*}{Dynamic} & Mixed  & 0.0905 & 0.8982 & 28.04 \\
        & Static  & 0.0905 & 0.8981 & 28.06 \\
        & Dynamic & 0.0921 & 0.8975 & 27.90 \\ 
        \bottomrule
    \end{tabular}
\end{table}

As shown in \cref{tab:comparison_prompt_distribution}, the performance variance across different training sources is remarkably marginal.
For instance, a policy trained on static prompts performs almost identically to one trained on dynamic prompts when tested on dynamic sequences (0.0905 vs. 0.0921 LPIPS).
This high degree of stability suggests that GCache captures fundamental structural redundancies within the diffusion process that are invariant to specific prompt semantics or motion levels, ensuring its robustness for diverse real-world applications.

\subsection{Generalization Across Resolutions}
\label{app:exp_results:Generalization_resolution}

To assess the spatial scalability of GCache, we evaluate the transferability of a policy optimized at a fixed resolution to unseen spatial scales.
Specifically, we apply the GCache-fast policy, originally optimized for $1024\times1024$ resolution on Flux-dev 1.0~\cite{flux2024}, directly to $512\times512$ and $256\times256$ settings without any further re-tuning or adaptation.

\begin{table}[tbp]
    \centering
    \caption{
        Zero-shot generalization across resolutions.
        We evaluate the GCache-fast policy (optimized at $1024 \times 1024$) on lower resolutions ($512$ and $256$) without re-tuning.
        GCache-fast consistently outperforms ERTACache across all scales, demonstrating its robustness to spatial resolution shifts.
    }
    \label{tab:resolution_comparison}
    \footnotesize
    \vskip 0.1in
    \begin{tabular}{ccccc}
        \toprule
        \textbf{Resolution} & \textbf{Method} & \textbf{LPIPS$\downarrow$} & \textbf{SSIM$\uparrow$} & \textbf{PSNR$\uparrow$} \\ 
        \midrule
        \multirow{2}{*}{1024} & ERTACache   & 0.2658 & 0.7863 & 20.60 \\
        & GCache-fast & \textbf{0.1825} & \textbf{0.8423} & \textbf{23.76} \\ 
        \midrule
        \multirow{2}{*}{512}  & ERTACache   & 0.2359 & 0.7580 & 19.97 \\
        & GCache-fast & \textbf{0.1514} & \textbf{0.8344} & \textbf{23.47} \\ 
        \midrule
        \multirow{2}{*}{256}  & ERTACache   & 0.2047 & 0.7209 & 19.70 \\
        & GCache-fast & \textbf{0.1558} & \textbf{0.7746} & \textbf{21.74} \\ 
        \bottomrule
    \end{tabular}
\end{table}

As shown in \cref{tab:resolution_comparison}, GCache-fast consistently outperforms the baseline ERTACache~\cite{DBLP:journals/corr25:ERTACache} by a significant margin across all tested resolutions.
Notably, the performance gain remains robust even when the resolution is quadrupled ($256 \to 1024$), underscoring that GCache captures resolution-agnostic redundancy patterns within the diffusion backbone.
This zero-shot transfer capability is highly desirable for practical deployment, as it eliminates the need for resolution-specific optimization.

\subsection{Effectiveness of the Learned Policy}
\label{app:exp_results:effect_policy}

A key question is whether GCache's performance gains stem from a superior reuse policy or simply from different computational budgets.
Some existing methods, such as ERTACache \cite{DBLP:journals/corr25:ERTACache}, employ an auxiliary light-weight model to rectify the errors introduced by cache reuse.
To ensure a fair comparison and isolate the impact of the policy itself, we evaluate GCache against two variants of ERTACache under identical budget $K$:
(1) ERTACache, the full version with its Time Adjustment and Error Rectification modules;
and (2) ERTACache*, a stripped-down version that excludes these additional modules, relying solely on its default reuse policy.

\begin{table}[tbp]
    \centering
    \caption{
        Policy comparison under identical budget $K$.
        We compare GCache with ERTACache and its variant ERTACache* (without the error rectification module).
        GCache consistently achieves superior performance without requiring any auxiliary corrector models.}
    \label{tab:same_k_comparison}
    \footnotesize
    \vskip 0.1in
    \begin{tabular}{cccccc}
        \toprule
        \textbf{Model} & \textbf{K} & \textbf{Method} & \textbf{LPIPS$\downarrow$} & \textbf{SSIM$\uparrow$} & \textbf{PSNR$\uparrow$} \\ 
        \midrule
        \multirow{3}{*}{OpenSora 1.2} & \multirow{3}{*}{12} & ERTACache* & 0.2198 & 0.7648 & 20.24 \\
        & & ERTACache  & 0.1659 & 0.8170 & 22.34 \\
        & & GCache     & \textbf{0.1188} & \textbf{0.8569} & \textbf{25.73} \\ 
        \midrule
        \multirow{3}{*}{CogVideoX-2B} & \multirow{3}{*}{17} & ERTACache* & 0.1090 & 0.8687 & 26.54 \\
        & & ERTACache  & 0.1012 & 0.8702 & 26.44 \\
        & & GCache     & \textbf{0.0721} & \textbf{0.9042} & \textbf{29.14} \\ 
        \bottomrule
    \end{tabular}
\end{table}

As summarized in \cref{tab:same_k_comparison}, GCache consistently outperforms both variants by a significant margin across multiple backbones.
Notably, GCache achieves substantially better LPIPS and PSNR than the full ERTACache, despite not using any extra corrector models.
For instance, on OpenSora 1.2, GCache improves PSNR from 22.34 to 25.73.
These results demonstrate that GCache’s optimization-based approach discovers a much more effective reuse trajectory, proving that a well-optimized policy can be more powerful than a sub-optimal policy supplemented by error correction.

\subsection{Validation of the Pre-computed Local Error Proxy}
\label{app:exp_results:precompute_error}

As discussed in \cref{app:training-details}, to circumvent the prohibitive computational overhead of generating online trajectories during optimization, GCache utilizes a pre-computed error matrix $E_{i,j} = \| \bm \delta_{t_i} - \bm \delta_{t_j} \|_1$ derived from the original (full-step) diffusion trajectory.
However, a potential concern is trajectory drift: as cache reuse is introduced, the intermediate latent states may deviate from the original path, potentially rendering the pre-computed errors inaccurate.

To investigate the fidelity of this approximation, we conduct an empirical study on CogVideoX-2B ($K=17$).
We compare the local residual error, defined as $\|\bm \delta^c - \bm \delta_{t_i}\|_1$, where $\bm \delta_{t_i}$ is the ground-truth residual at timestep $t_i$.
We calculate this error under two settings:
(1) the Original Trajectory, where $\bm \delta^c$ is sampled from the full-step inference;
and (2) the Cache-reused Trajectory, where $\bm \delta^c$ is sampled from an actual inference process governed by our learned GCache policy.

\begin{figure}[tbp]
    \centering
    \includegraphics[width=0.7\textwidth]{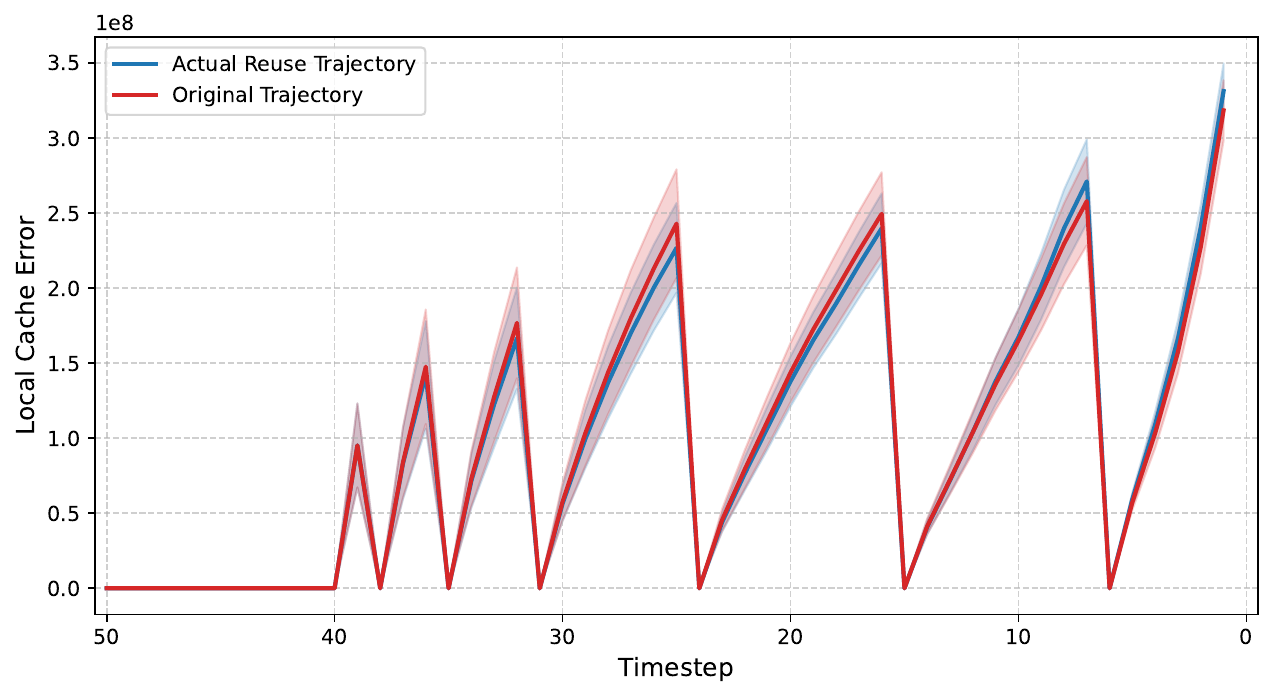}
    \caption{
        Validation of the pre-computed error proxy.
        We compare the L1 residual error $\|\bm \delta^c - \bm \delta_{t_i}\|_1$ during the denoising process on CogVideoX-2B ($K=17$).
        The Pre-compute (red) curve denotes the error calculated from the original trajectory, while the Cache (blue) curve represents the actual error in a trajectory with GCache reuse.
        The ``sawtooth'' peaks correspond to the $K$ cache refresh points.
        The high degree of overlap between the two curves justifies the use of pre-computed error $E$ as a high-fidelity and efficient proxy for optimization.
    }
    \label{fig:precompute_error}
\end{figure}

The results are visualized in \cref{fig:precompute_error}.
As shown, the error profiles of the two trajectories are remarkably aligned across the entire denoising process.
The ``sawtooth'' pattern reflects the periodic refreshing of the cache, where the error peaks just before a cache update and drops to zero immediately after.
Crucially, the mean and variance of the errors in the cache-reused trajectory (blue) closely track those of the pre-computed proxy (red), even in the late stages of denoising where drift is typically most pronounced.
This high degree of alignment justifies the use of pre-computed errors as a reliable and efficient proxy for optimization, as it accurately reflects the error dynamics of the actual inference process.

\section{Additional Qualitative Comparison}
\label{app:qual_comparison}

We provide additional qualitative results for both image and video generation.

\subsection{Image Generation}

For image generation, \cref{fig:img_app} presents additional qualitative comparisons on Flux-dev~1.0.
Under the same acceleration setting, ERTACache introduces various visual degradations, including background misalignment, incorrect spatial relationships between objects and subjects, missing or distorted materials, and semantic artifacts such as erroneous objects (e.g., incorrect traffic signs) and physically implausible details (e.g., airplanes depicted mid-air with landing gear deployed).
By contrast, GCache remains highly faithful to the original model outputs, preserving semantic correctness, spatial coherence, and fine-grained visual details even under a $2.87\times$ speedup setting.

\subsection{Vedio Generation}

For video generation, we provide additional qualitative comparisons across three representative text-to-video diffusion models, including CogVideoX-2B (\cref{fig:cog_temporal}), Open-Sora~1.2 (\cref{fig:opensora_temporal}), and Wan~2.1-1.3B (\cref{fig:wan_temporal}).
For each example, we uniformly sample six frames from the generated video sequence to visualize temporal consistency and structural fidelity over time.
Within each group, we compare outputs from ERTACache, GCache, and the ground-truth full-step original generation.
Across all three models and diverse prompts, GCache consistently produces videos that remain highly faithful to the ground-truth outputs, preserving semantic correctness, object structure, and temporal coherence throughout the sequence.
The generated contents exhibit stable object appearances and natural motion transitions, with minimal degradation under accelerated sampling.
In contrast, ERTACache frequently introduces noticeable visual artifacts and semantic inconsistencies.
Typical failure cases include distorted object geometry (e.g., malformed cups, umbrellas, and sharks), incorrect semantic structures (e.g., unrealistic astronaut body configurations), and unstable object appearances across frames.
These artifacts become more evident in dynamic scenes and complex compositions, indicating weaker preservation of global semantic information.
Overall, the qualitative results further demonstrate that GCache achieves substantially better fidelity and temporal consistency under the same acceleration budget.

\section{Limitations}
\label{app:limitations}

Although GCache effectively optimizes cache policies, it currently operates under a fixed refresh budget and relies on pre-computed error proxies, which may not fully account for dynamic input complexity or extreme trajectory shifts.
Future work could explore sample-adaptive scheduling and more diverse perceptual objectives to further enhance temporal consistency in highly dynamic videos.

\section{Broader Impact}
\label{app:broader_impact}

By substantially reducing the computational cost of large-scale diffusion models, GCache promotes environmental sustainability and democratizes access to state-of-the-art generative tools for users with limited hardware resources.
While accelerated generation could potentially be misused for creating synthetic misinformation, we advocate for its deployment alongside robust safety filters and digital watermarking technologies to mitigate such risks.

\begin{figure}[hp]
    \centering
    \includegraphics[width=\textwidth]{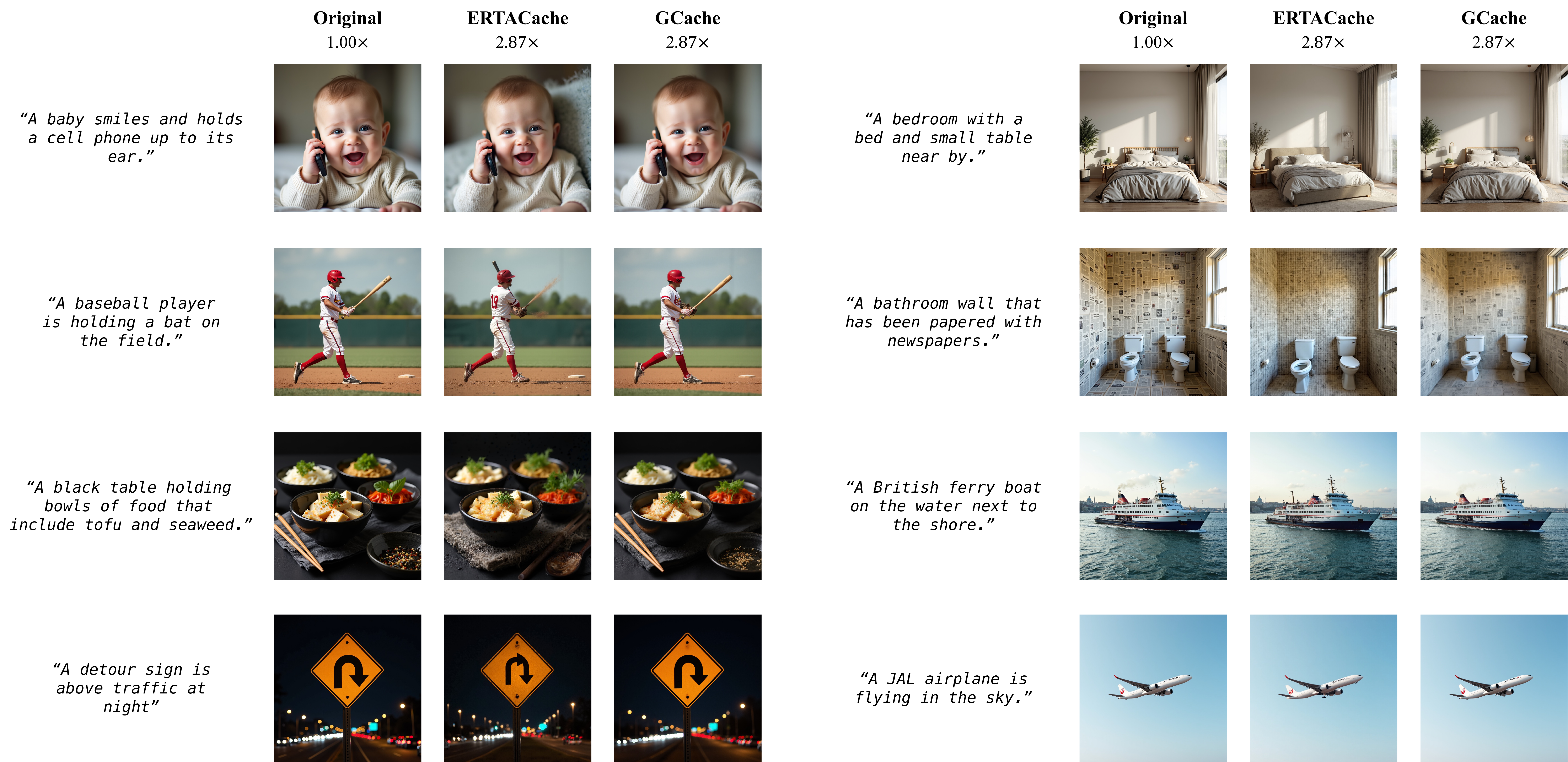}
    \caption{Additional qualitative comparison results for image generation on Flux-dev 1.0. Best viewed zoomed in.}
    \label{fig:img_app}
\end{figure}

\vspace*{\fill}
\begin{figure*}[p]
\centering
\footnotesize
\setlength{\tabcolsep}{2pt}

\begin{tabular}{>{\centering\arraybackslash}m{0.12\textwidth} m{0.84\textwidth}}

\textbf{ERTACache} &
\includegraphics[width=\linewidth]{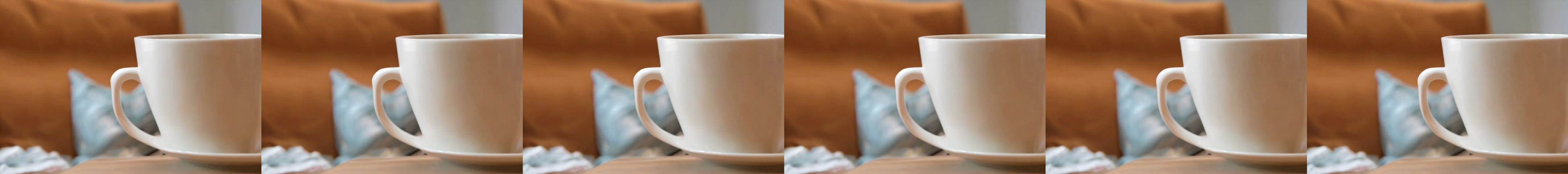} \\

\textbf{GCache} &
\includegraphics[width=\linewidth]{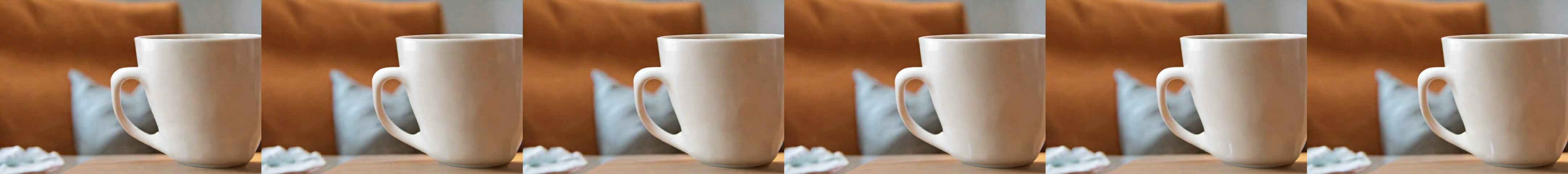} \\

\textbf{Ground Truth} &
\includegraphics[width=\linewidth]{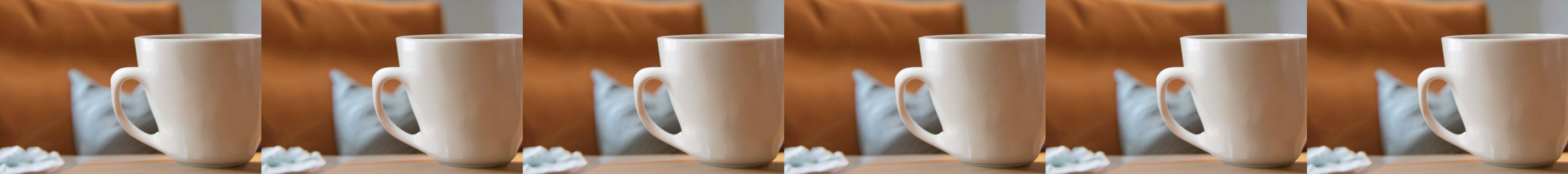} \\

&
\multicolumn{1}{c}{\textbf{(a)} \textit{A cup and a couch.}} \\

\addlinespace[6pt]

\textbf{ERTACache} &
\includegraphics[width=\linewidth]{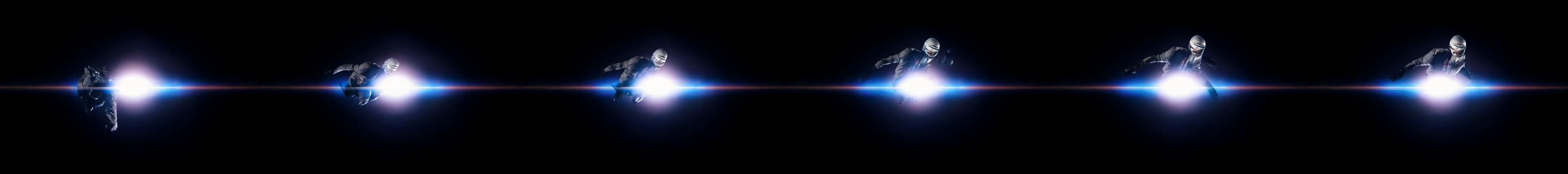} \\

\textbf{GCache} &
\includegraphics[width=\linewidth]{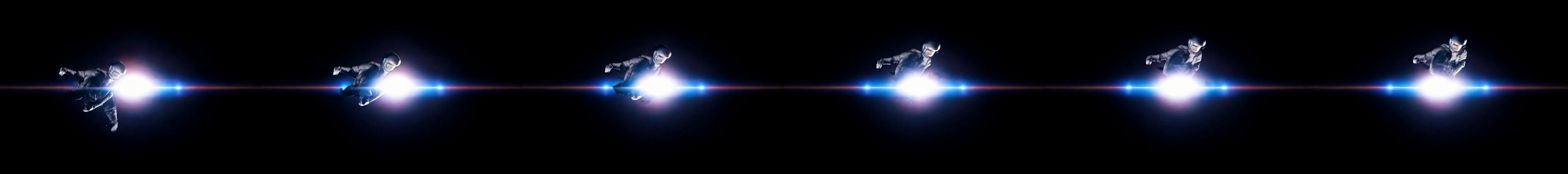} \\

\textbf{Ground Truth} &
\includegraphics[width=\linewidth]{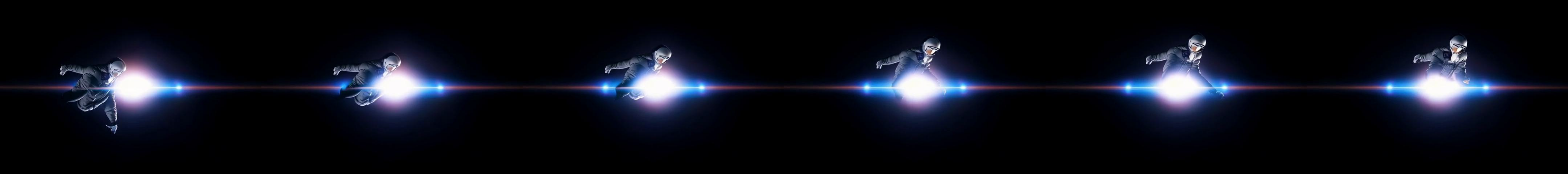} \\

&
\multicolumn{1}{c}{\textbf{(b)} \textit{An astronaut flying in space, in super slow motion.}} \\

\addlinespace[6pt]

\textbf{ERTACache} &
\includegraphics[width=\linewidth]{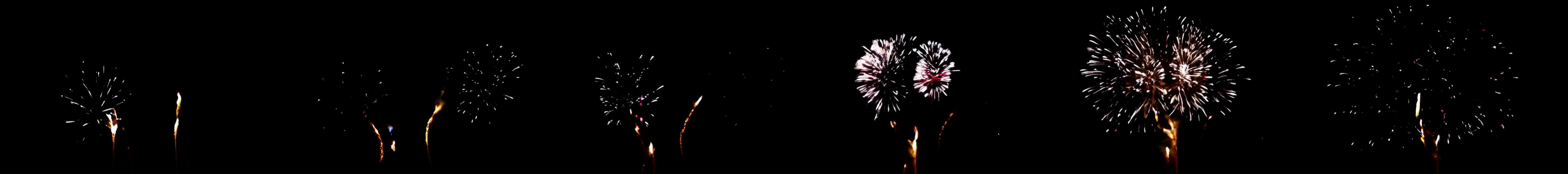} \\

\textbf{GCache} &
\includegraphics[width=\linewidth]{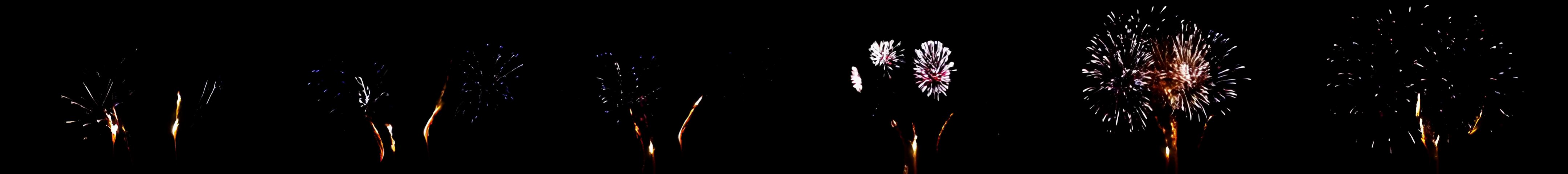} \\

\textbf{Ground Truth} &
\includegraphics[width=\linewidth]{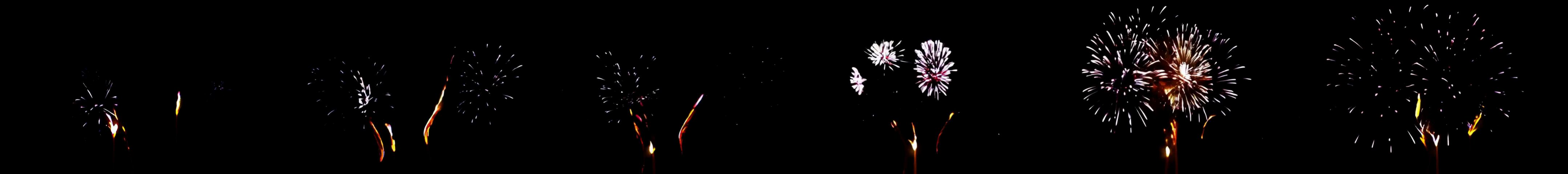} \\

&
\multicolumn{1}{c}{\textbf{(c)} \textit{Fireworks.}} \\

\addlinespace[6pt]

\textbf{ERTACache} &
\includegraphics[width=\linewidth]{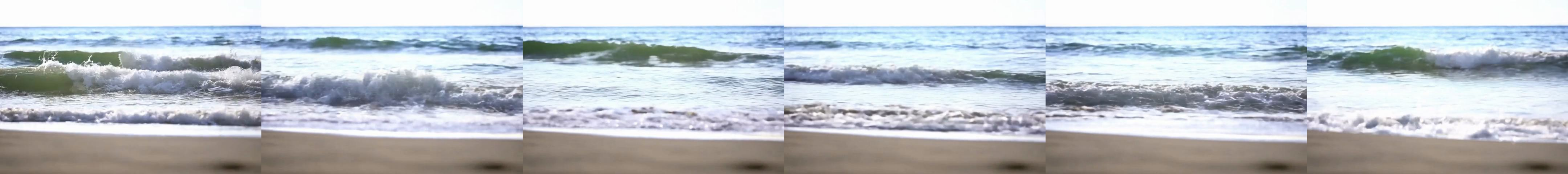} \\

\textbf{GCache} &
\includegraphics[width=\linewidth]{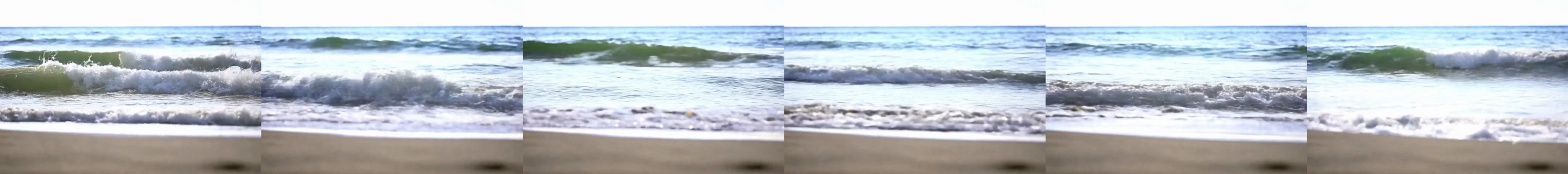} \\

\textbf{Ground Truth} &
\includegraphics[width=\linewidth]{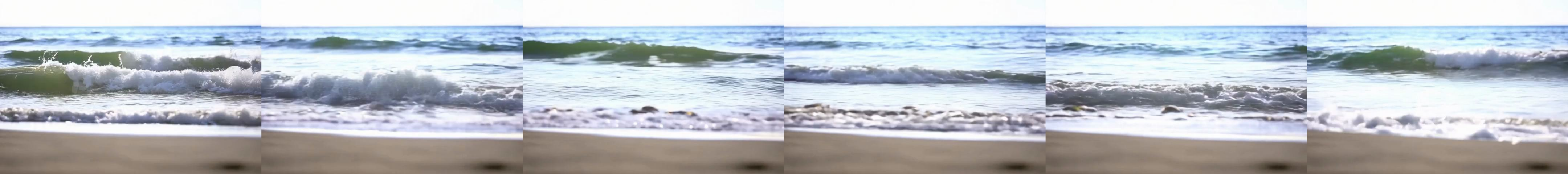} \\

&
\multicolumn{1}{c}{\textbf{(d)} \textit{A beautiful coastal beach in spring, waves lapping on sand, racking focus.}} \\

\end{tabular}

\caption{
Temporal consistency comparison on \textbf{CogVideoX-2B}.
Each group visualizes six evenly spaced frames from the generated sequence.
}
\label{fig:cog_temporal}
\end{figure*}

% \vspace*{\fill}

% \clearpage

% \vspace*{\fill}

\begin{figure*}[p]
\centering
\footnotesize
\setlength{\tabcolsep}{2pt}

\begin{tabular}{>{\centering\arraybackslash}m{0.12\textwidth} m{0.84\textwidth}}

\textbf{ERTACache} &
\includegraphics[width=\linewidth]{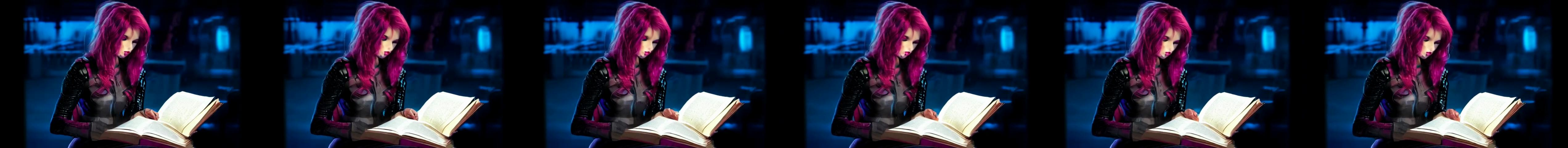} \\

\textbf{GCache} &
\includegraphics[width=\linewidth]{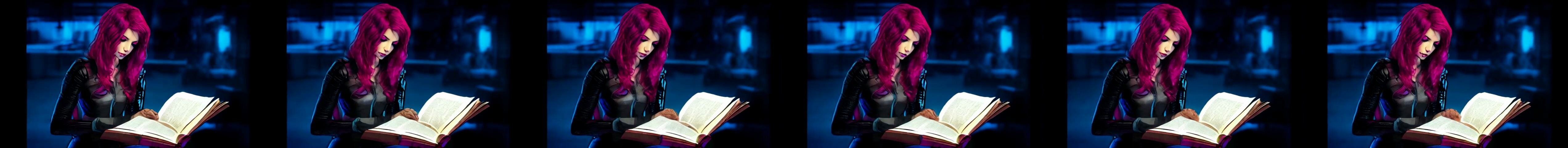} \\

\textbf{Ground Truth} &
\includegraphics[width=\linewidth]{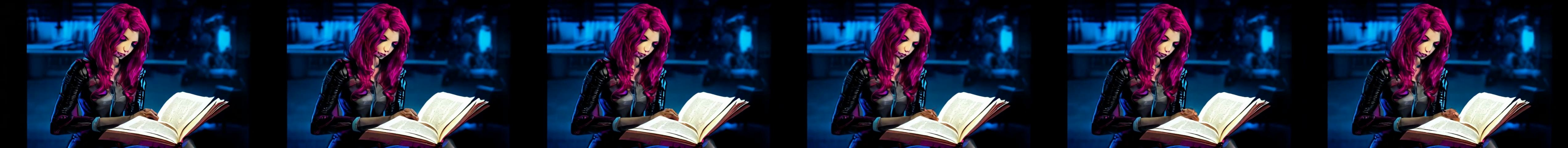} \\

&
\multicolumn{1}{c}{\textbf{(a)} \textit{Gwen Stacy reading a book, in cyberpunk style.}} \\

\addlinespace[6pt]

\textbf{ERTACache} &
\includegraphics[width=\linewidth]{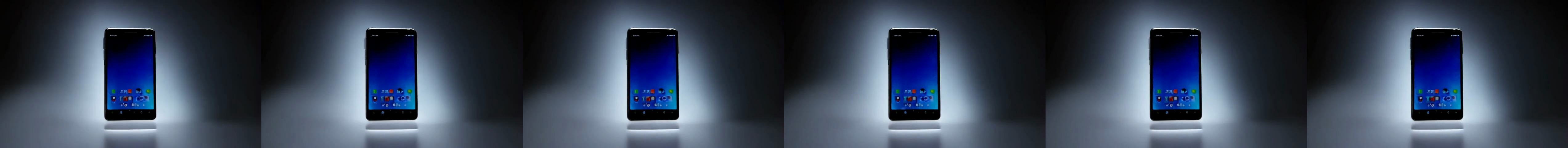} \\

\textbf{GCache} &
\includegraphics[width=\linewidth]{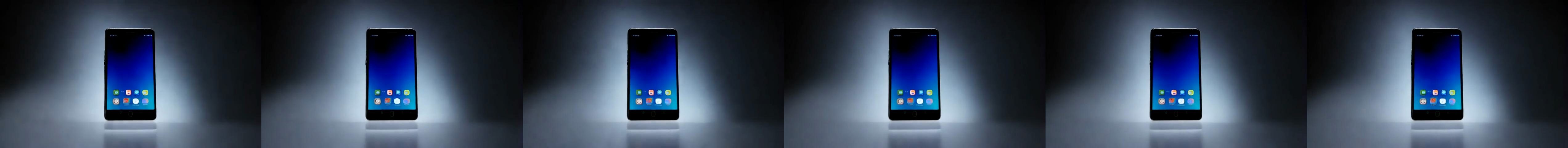} \\

\textbf{Ground Truth} &
\includegraphics[width=\linewidth]{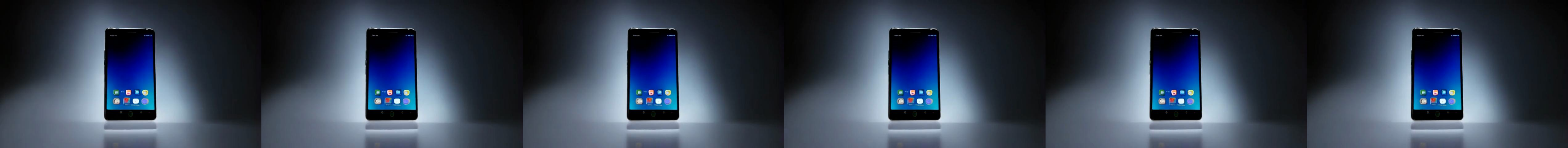} \\

&
\multicolumn{1}{c}{\textbf{(b)} \textit{a cell phone.}} \\

\addlinespace[6pt]

\textbf{ERTACache} &
\includegraphics[width=\linewidth]{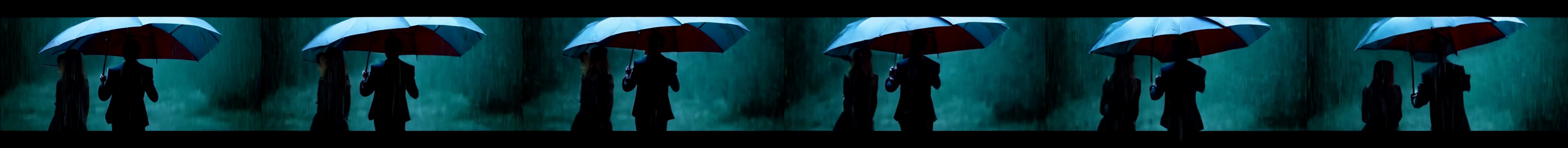} \\

\textbf{GCache} &
\includegraphics[width=\linewidth]{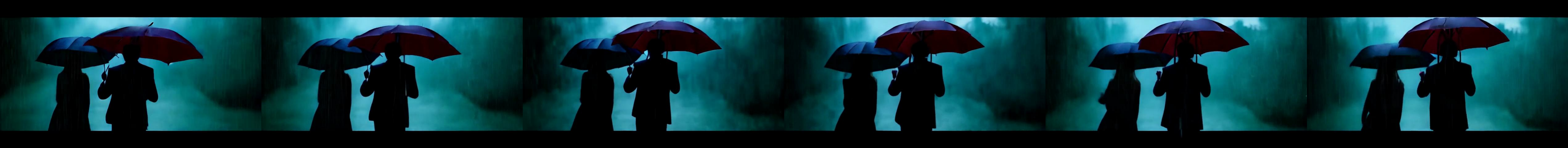} \\

\textbf{Ground Truth} &
\includegraphics[width=\linewidth]{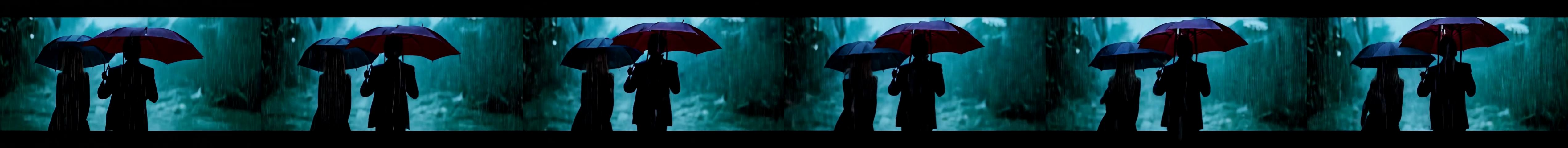} \\

&
\multicolumn{1}{c}{
\shortstack[c]{
\textbf{(c)} \textit{A couple in formal evening wear going home get caught in a} \\
\textit{heavy downpour with umbrellas, in super slow motion.}
}
} \\

\addlinespace[6pt]

\textbf{ERTACache} &
\includegraphics[width=\linewidth]{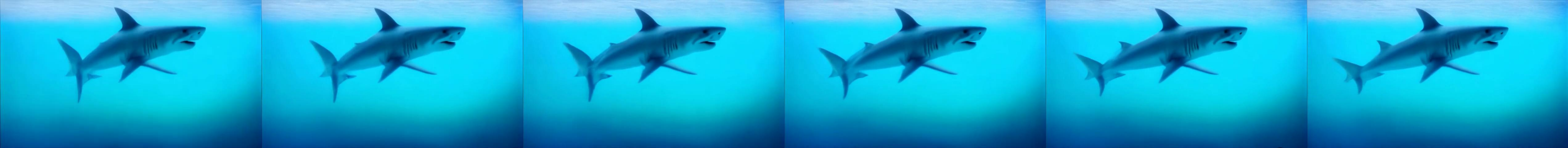} \\

\textbf{GCache} &
\includegraphics[width=\linewidth]{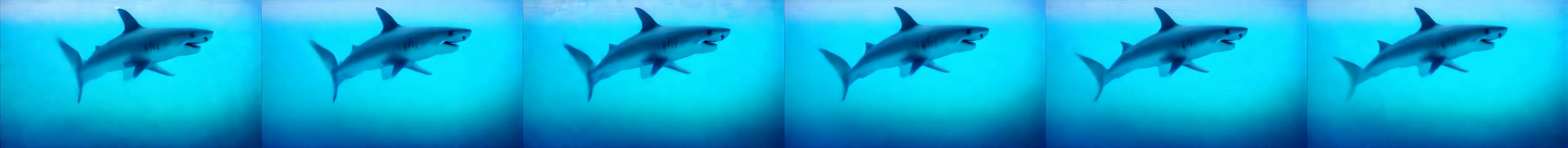} \\

\textbf{Ground Truth} &
\includegraphics[width=\linewidth]{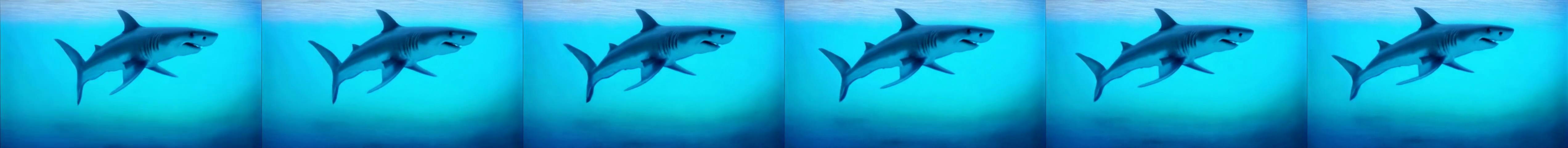} \\

&
\multicolumn{1}{c}{\textbf{(d)} \textit{a shark is swimming in the ocean, oil painting}} \\

\end{tabular}

\caption{
Temporal consistency comparison on \textbf{Open-Sora~1.2}.
Each group visualizes six evenly spaced frames from the generated sequence.
}
\label{fig:opensora_temporal}
\end{figure*}

\begin{figure*}[p]
\centering
\footnotesize
\setlength{\tabcolsep}{2pt}

\begin{tabular}{>{\centering\arraybackslash}m{0.12\textwidth} m{0.84\textwidth}}

\textbf{ERTACache} &
\includegraphics[width=\linewidth]{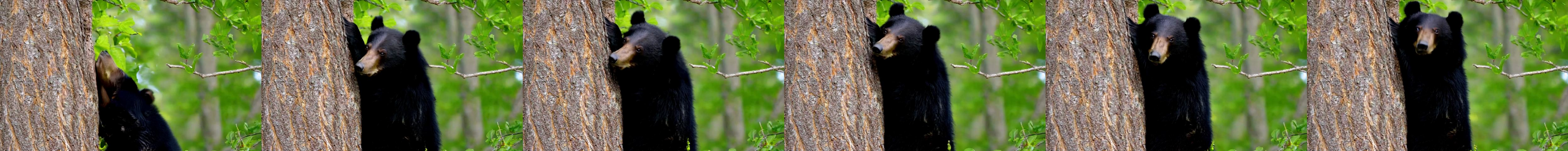} \\

\textbf{GCache} &
\includegraphics[width=\linewidth]{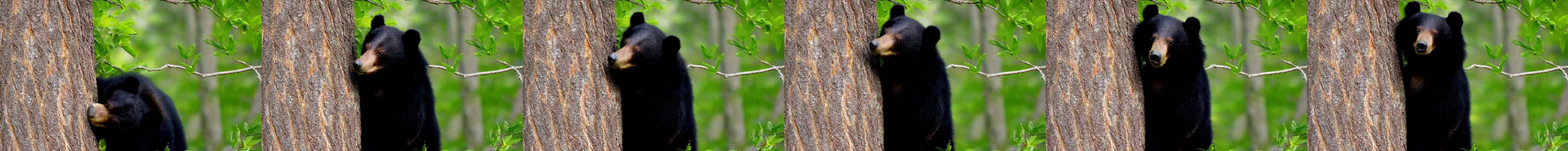} \\

\textbf{Ground Truth} &
\includegraphics[width=\linewidth]{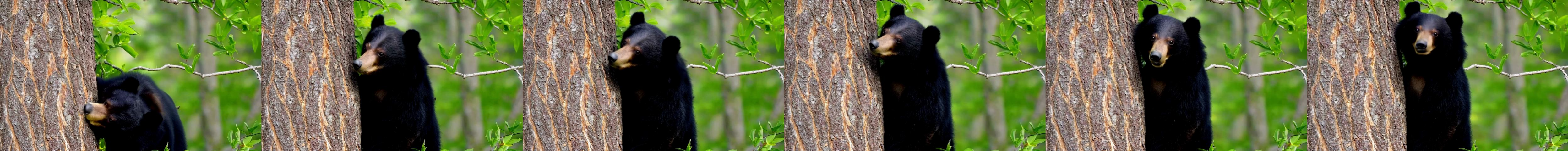} \\

&
\multicolumn{1}{c}{
\textbf{(a)} \textit{A bear climbing a tree.}
} \\

\addlinespace[6pt]

\textbf{ERTACache} &
\includegraphics[width=\linewidth]{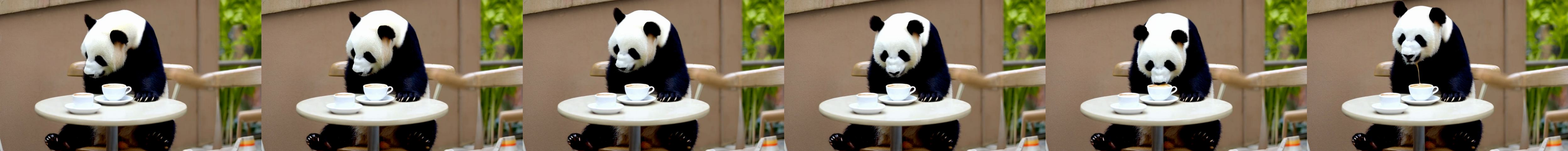} \\

\textbf{GCache} &
\includegraphics[width=\linewidth]{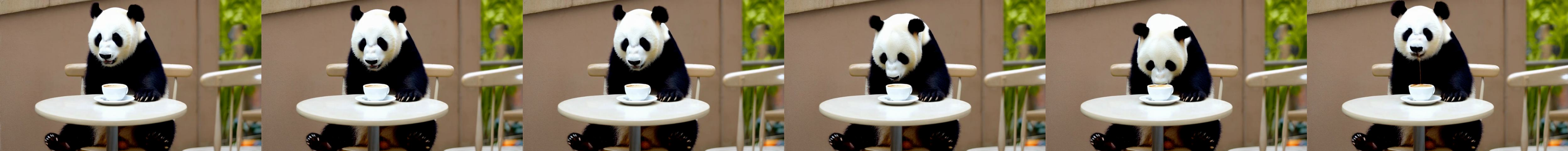} \\

\textbf{Ground Truth} &
\includegraphics[width=\linewidth]{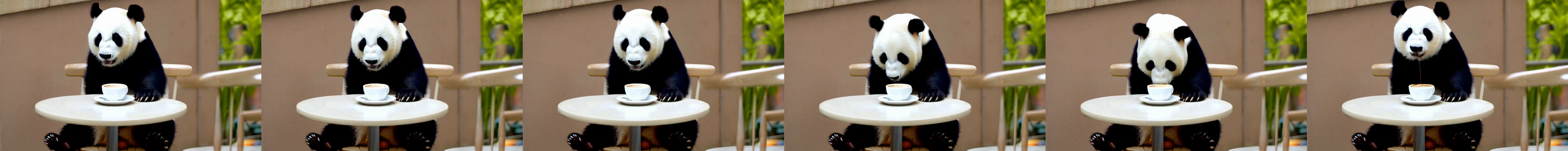} \\

&
\multicolumn{1}{c}{
\textbf{(b)} \textit{A panda drinking coffee in a cafe in Paris, pan left.}
} \\

\addlinespace[6pt]

\textbf{ERTACache} &
\includegraphics[width=\linewidth]{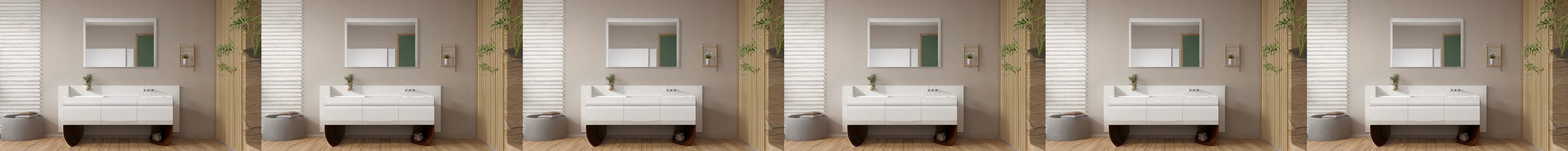} \\

\textbf{GCache} &
\includegraphics[width=\linewidth]{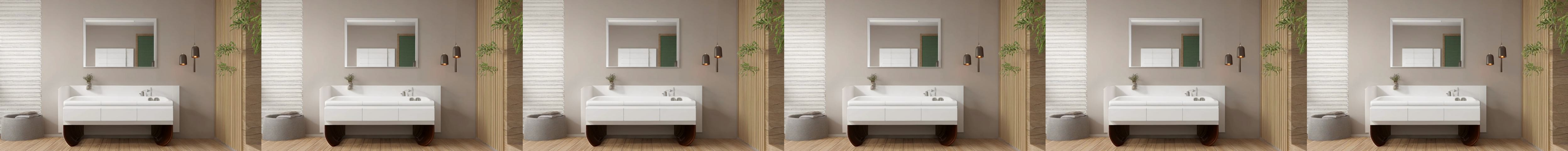} \\

\textbf{Ground Truth} &
\includegraphics[width=\linewidth]{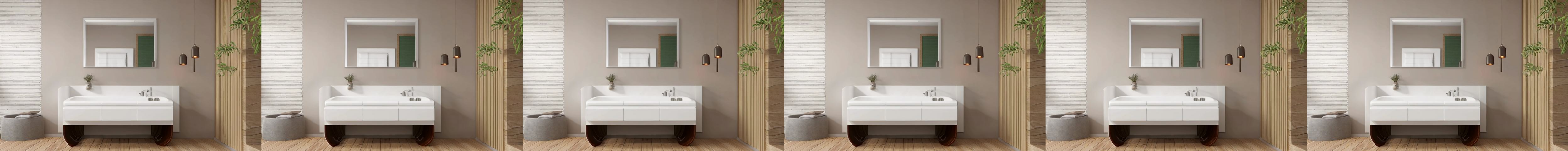} \\

&
\multicolumn{1}{c}{
\textbf{(c)} \textit{A tranquil tableau of bathroom.}
} \\

\addlinespace[6pt]

\textbf{ERTACache} &
\includegraphics[width=\linewidth]{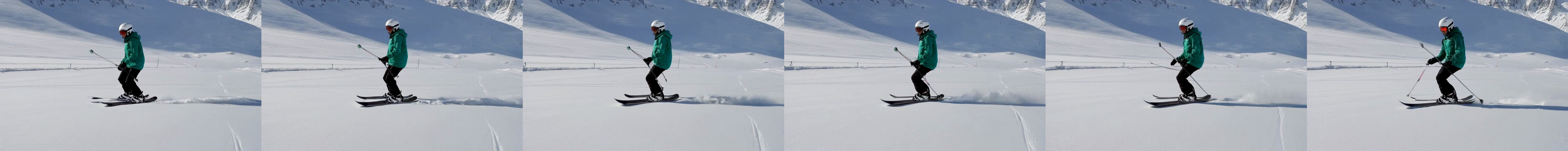} \\

\textbf{GCache} &
\includegraphics[width=\linewidth]{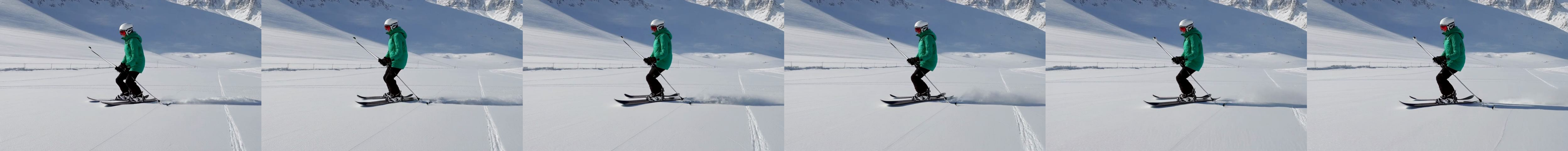} \\

\textbf{Ground Truth} &
\includegraphics[width=\linewidth]{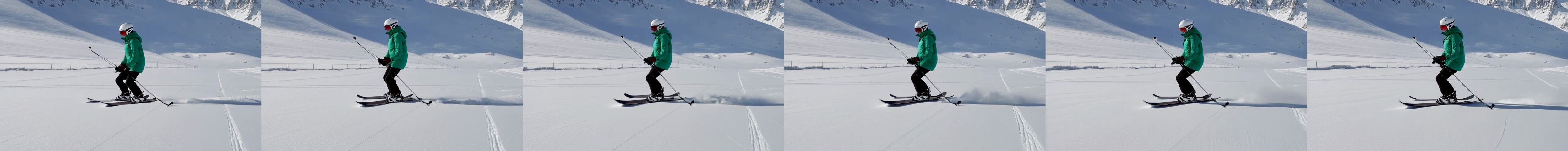} \\

&
\multicolumn{1}{c}{
\textbf{(d)} \textit{Skis and a snowboard.}
} \\

\end{tabular}

\caption{
Temporal consistency comparison on \textbf{Wan~2.1-1.3B}.
Each group visualizes six evenly spaced frames from the generated sequence.
}
\label{fig:wan_temporal}
\end{figure*}

% \end{appendices}

% \input{appendix}

%%%%%%%%%%%%%%%%%%%%%%%%%%%%%%%%%%%%%%%%%%%%%%%%%%%%%%%%%%%%

% \clearpage
% \newpage
% \input{checklist.tex}

\end{document}